\documentclass[twoside]{article}

\usepackage[accepted]{aistats2019}
\usepackage{natbib}
\usepackage{amsthm,amsmath,amsfonts,natbib,graphicx,bm,booktabs, soul, cancel, mathtools, float}
\usepackage{mathabx}

\newtheorem{asmp}{Assumption}[section]
\newtheorem{thm}{Theorem}
\newtheorem{coro}{Corollary}

\newtheorem{remark}{Remark}

\newtheorem{defn}{Definition}[section]

\DeclareMathOperator*{\diag}{diag}

\DeclareMathOperator*{\E}{\mathbb{E}}

\begin{document}



\twocolumn[

\aistatstitle{Interaction Matters: 
A Note on Non-asymptotic Local Convergence of Generative Adversarial Networks}

\aistatsauthor{ Tengyuan Liang \And James Stokes }

\aistatsaddress{ University of Chicago, Booth School of Business \And  University of Pennsylvania }

]

\begin{abstract}
Motivated by the pursuit of a systematic computational and algorithmic understanding of Generative Adversarial Networks (GANs), we present a simple yet unified non-asymptotic local convergence theory for smooth two-player games, which subsumes several discrete-time gradient-based saddle point dynamics. The analysis reveals the surprising nature of the off-diagonal interaction term as both a blessing and a curse. On the one hand, this interaction term explains the origin of the slow-down effect in the convergence of Simultaneous Gradient Ascent (SGA) to stable Nash equilibria. On the other hand, for the unstable equilibria, exponential convergence can be proved thanks to the interaction term, for four modified dynamics proposed to stabilize GAN training: Optimistic Mirror Descent (OMD), Consensus Optimization (CO), Implicit Updates (IU) and Predictive Method (PM). The analysis uncovers the intimate connections among these stabilizing techniques, and provides detailed characterization on the choice of learning rate. As a by-product, we present a new analysis for OMD proposed in \citet*{daskalakis2017training} with improved rates.
\end{abstract}

\section{Introduction}

In this paper we consider the non-asymptotic local convergence and stability of discrete-time gradient-based optimization algorithms for solving smooth two-player zero-sum games of the form,
\begin{equation}\label{e:minimax}
	\min_{\theta \in \mathbb{R}^p}\max_{\omega \in \mathbb{R}^q} U(\theta, \omega) \enspace .
\end{equation}
The motivation behind our non-asymptotic analysis follows from the observation that Generative Adversarial Networks (GANs) lack principled understanding at both the computational and algorithmic level. % \footnote{We refer the readers to \cite{ferencefixgan}'s concise post on this issue.}
GAN optimization is a special case of \eqref{e:minimax}, which has been developed for learning a complex and multi-modal probability distribution based on samples from $\mathcal{P}_{\rm real}$ (over $\mathcal{X}$),  through learning a generator function $g_\theta(\cdot)$ that transforms the input distribution $\mathcal{P}_{\rm input}$ (over $\mathcal{Z}$) to match the target $\mathcal{P}_{\rm real}$. Ignoring the parameter regularization, the value function corresponding to a GAN is of the form,
\begin{align}\label{e:gan}
	U(\theta, \omega) =  \E h_1(f_{\omega}(X))- \E h_2(f_{\omega}(g_{\theta}(Z))) \enspace ,
\end{align}
where $X \sim \mathcal{P}_{\rm real}$, $Z \sim \mathcal{P}_{\rm input}$ and $(\theta, \omega) \in \mathbb{R}^p \times \mathbb{R}^q$ parametrizes the generator function $g_\theta : \mathcal{Z} \to \mathcal{X}$ and discriminator function $f_\omega : \mathcal{X} \to \mathbb{R}$, respectively.  The original GAN \citep{goodfellow2014generative}, for example, corresponds to choosing $h_1(t) = \log \sigma(t)$, $h_2(t) = -\log(1-\sigma(t))$ where $\sigma$ is the sigmoid function; Wasserstein GAN \citep{arjovsky2017wasserstein} considers $h_1(t) = h_2(t) = t$;  $f$-GAN \citep{nowozin2016f} proposes to use $h_1(t) = t, h_2(t) = f^*(t)$, where $f^*$ denotes the Fenchel dual of $f$. Recently, several attempts have been made to understand whether GANs learn the target distribution in the statistical sense \citep{liu2017approximation,arora2017gans,liang2017well, liang2018well, arora2017theoretical, liu2018inductive}.

Optimization of GANs (and value functions of the form \eqref{e:minimax} at large) is hard, both in theory and in practice \citep{singh2000nash, pfau2016connecting, salimans2016improved}. 
Global optimization of a general value function with multiple saddle points is impractical and unstable, so we instead resort to the more modest problem of searching for a \textit{local saddle point} $(\theta_*, \omega_*)$ such that no player has the incentive to deviate locally
\begin{align*}
	&U(\theta_*, \omega_*) \leq U(\theta, \omega_*) \enspace , ~~\text{for $\theta$ in an open nbhd of $\theta_*$} \enspace , \\
	&U(\theta_*, \omega_*) \geq U(\theta_*, \omega) \enspace , ~~\text{for $\omega$ in an open nbhd of $\omega_*$} \enspace .
\end{align*}
For smooth value functions, the above conditions are equivalent to the following solution concept:
\begin{defn}[Local Nash Equilibrium]
	\label{def:lne}
	$(\theta_*, \omega_*)$ is called a local Nash equilibrium if 
	\begin{enumerate}
		\item $\nabla_\theta U(\theta_*, \omega_*) = \mathbf{0}$, $\nabla_\omega U(\theta_*, \omega_*) = \mathbf{0}$; 
		\item $\nabla_{\theta \theta} U(\theta_*, \omega_*) \succeq \mathbf{0}$, $- \nabla_{\omega \omega} U(\theta_*, \omega_*) \succeq \mathbf{0}$.
	\end{enumerate}
\end{defn}
Here we use $\nabla_{\theta \omega} U(\theta, \omega)\in \mathbb{R}^{p \times q}$ to denote the off-diagonal term $\partial^2 U/\partial \theta \partial \omega$, and name it the \textit{interaction term} throughout the paper. $\nabla_\theta U(\theta, \omega) \in \mathbb{R}^p$ denotes the gradient $\partial U/\partial \theta$, and $\nabla_{\theta\theta} U(\theta, \omega) \in \mathbb{R}^{p\times p}$ for the Hessian $\partial^2 U/\partial \theta \partial \theta$.

In practice, discrete-time dynamical systems are employed to numerically approach the saddle points of $U(\theta, \omega)$, as is the case in GANs \citep{goodfellow2014generative}, and in primal-dual methods for non-linear optimization \citep{singh2000nash}. The simplest possibility is \textit{Simultaneous Gradient Ascent} (SGA), which corresponds to the following discrete-time dynamical system,
\begin{align}\label{e:sga}
	\theta_{t+1} &= \theta_t - \eta \nabla_{\theta} U(\theta_t, \omega_t) \enspace , \nonumber \\
	\omega_{t+1} &= \omega_t + \eta \nabla_{\omega} U(\theta_t, \omega_t) \enspace ,
\end{align}
where $\eta$ is the step size or learning rate. 
In the limit of vanishing step size, SGA approximates a continuous-time autonomous dynamical system, the asymptotic convergence of which has been established in \cite{singh2000nash, cherukuri2017saddle, nagarajan2017gradient}. In practice, however, it has been widely reported that the discrete-time SGA dynamics for GAN optimization suffers from instabilities due to the possibility of complex eigenvalues in the operator of the dynamical system \citep{salimans2016improved, metz2016unrolled, nagarajan2017gradient, mescheder2017numerics,heusel2017gans}.
We believe room for improvement still exists in the current theory, which we hope will render it to be more informative in practice:
\begin{itemize}
	\item \textbf{Non-asymptotic convergence speed.} \quad In practice, one is concerned with finite step size $\eta > 0$ which is typically subject to extensive hyperparameter tuning. Detailed characterizations on the convergence speed, and theoretical insights on the choice of learning rate can be helpful. 
	\item \textbf{Unified simple analysis for modified saddle point dynamics.} \quad Several attempts to fix GAN optimization have been put forth by independent researchers, which modify the dynamics \citep{mescheder2017numerics, daskalakis2017training, yadav2017stabilizing} using very different insights. A unified analysis that reviews the deeper connections amongst these proposals helps to better understand the saddle point dynamics at large.  
\end{itemize}

In this paper, we address the above points by studying the theory of non-asymptotic convergence of SGA and related discrete-time saddle point dynamics, namely, \textit{Optimistic Mirror Descent} (OMD), \textit{Consensus Optimization} (CO), Implicit Updates (IU) and \textit{Predictive Method} (PM). More concretely, we provide the following theoretical contributions about the crucial effect of the off-diagonal interaction term $\nabla_{\theta \omega}U(\theta, \omega)$ in two-player games:
\begin{itemize}
	\item \textbf{Stable case: curse of the interaction term.} \quad Locally, SGA converges \emph{exponentially} fast to a stable Nash equilibrium with a carefully chosen learning rate. This can be viewed as a generalization (rather than a special case) of the local convergence guarantee for single-player gradient descent for strongly-convex functions. In addition, we quantitatively isolate the slow-down in the convergence rate of two-player SGA compared to single-player gradient descent, due to the presence of the off-diagonal interaction term $\nabla_{\theta \omega}U(\theta, \omega)$ for the two-player game. 
	
	\item \textbf{Unstable case: blessing of the interaction term.} \quad For unstable Nash equilibria, SGA \emph{diverges} away for any non-zero learning rate. We discover a unified non-asymptotic analysis that encompasses four proposed modified dynamics\,---\, OMD, CO, IU and PM. The analysis shows that all these algorithms, at a high level, share the same idea of utilizing the curvature introduced by the interaction term $\nabla_{\theta \omega}U(\theta, \omega)$. Unlike the slow sub-linear rate of convergence experienced by single-player gradient descent for non-strongly convex functions\footnote{In fact, \cite{nesterov2013introductory} constructed a convex function that is non-strongly convex, such that all first order methods suffer slow sub-linear rate of convergence (in optimization literature, linear rate refers to exponential convergence speed).}, four modified dynamics effectively exploit the interaction term to achieve \emph{exponential} convergence to unstable Nash equilibria. The analysis also provides specific advice on the choice of learning rate for each procedure, albeit restricted to the simple case of bi-linear games.
\end{itemize}

The organization of the paper is as follows. In Section \ref{sec:convex-concave} we consider the situation when locally, the value function satisfies strict convexity/concavity. We show non-asymptotic exponential convergence to Nash equilibria for SGA, and identify an optimal learning rate.
To reveal and understand the new features of the modified dynamics, we study in Section \ref{sec:bilinear} the minimal unstable bilinear game, showing that proposed stabilizing techniques all achieve exponential convergence to unstable Nash equilibria.
Finally, in Section \ref{sec:experiments} we take a step closer to the real world by numerically evaluating each of the proposed dynamical systems using value functions of GAN form \eqref{e:gan}, under objective evaluation metrics. Proofs are deferred to the Appendix \ref{sec:Appendix}.

\section{Stable Case: Non-asymptotic Local Convergence}\label{sec:convex-concave}

In this section we will establish the non-asymptotic convergence of SGA dynamics to saddle points that are \textit{stable local Nash equilibria}. With a properly chosen learning rate, the local convergence can be intuitively pictured as cycling inwards to these saddle points, where the distance to the saddle point of interest is exponentially contracting. First, let's introduce the notion of stable equilibrium.
 
\begin{defn}[Stable Local Nash Equilibrium]
	\label{def:slne}
	$(\theta_*, \omega_*)$ is called a stable local Nash equilibrium if 
	\begin{enumerate}
		\item $\nabla_\theta U(\theta_*, \omega_*) = \mathbf{0}$, $\nabla_\omega U(\theta_*, \omega_*) = \mathbf{0}$; 
		\item $\nabla_{\theta \theta} U(\theta_*, \omega_*) \succ \mathbf{0}$, $- \nabla_{\omega \omega} U(\theta_*, \omega_*) \succ \mathbf{0}$.
	\end{enumerate}
\end{defn}
The above notion of stability is stronger than the Definition~\ref{def:lne}, in the sense that $\nabla_{\theta \theta} U(\theta_*, \omega_*)$ and $- \nabla_{\omega \omega} U(\theta_*, \omega_*)$ have smallest eigenvalues bounded away from $0$. 

\begin{asmp}[Local Strong Convexity-Concavity]
	\label{asmp:convex-concave}
	Consider $U(\theta, \omega): \mathbb{R}^p \times \mathbb{R}^q \rightarrow \mathbb{R}$ that is smooth and twice differentiable, and let $(\theta_*, \omega_*)$ be a stable local Nash equilibrium as in Definition~\ref{def:slne}. Assume that for some $r>0$, there exists an open neighborhood near $(\theta_*, \omega_*)$ such that for all $(\theta, \omega) \in B_2 \left( (\theta_*, \omega_*), r \right)$, the following strong convexity-concavity condition holds,
	\begin{align*}
		\nabla_{\theta \theta} U(\theta, \omega) \succ \mathbf{0}, ~ - \nabla_{\omega \omega} U(\theta, \omega) \succ \mathbf{0} \enspace .
	\end{align*}
\end{asmp}

It will prove convenient to introduce some notation before introducing the main theorem. Let us define the following block-wise abbreviation for the matrix of second derivatives,
\begin{align}
	\label{eq:abrev}
	\begin{bmatrix}
		\nabla_{\theta\theta} U( \theta, \omega) & \nabla_{\theta \omega}U(\theta, \omega) \\
		-\nabla_{\omega \theta} U( \theta, \omega) & -\nabla_{\omega \omega} U( \theta, \omega) 
	\end{bmatrix} := 
	\begin{bmatrix}
			A_{\theta, \omega} & C_{\theta, \omega}\\
			-C_{\theta, \omega}^T & B_{\theta, \omega}
\end{bmatrix} \enspace ,
\end{align}
and define $\alpha, \beta$ as
\begin{align}
	\label{eq:alpha_beta}
	\alpha &:= \min_{(\theta, \omega) \in B_2((\theta_*, \omega_*), r)}\lambda_{\min}\left( \diag(A_{\theta, \omega}^2 , B_{\theta, \omega}^2) \right) \enspace , \nonumber\\
	\beta &:= \max_{(\theta, \omega) \in B_2((\theta_*, \omega_*), r)}\lambda_{\max}\left(F_{\theta,\omega}  \right) \enspace , \\
	F_{\theta,\omega} &:= \nonumber \\
	 & \begin{bmatrix}
			A_{\theta, \omega}^2 + C_{\theta, \omega}C_{\theta, \omega}^T & -A_{\theta, \omega}C_{\theta, \omega}+C_{\theta, \omega}B_{\theta, \omega} \\
			-C_{\theta, \omega}^TA_{\theta, \omega}+B_{\theta, \omega}C^T_{\theta, \omega} & B_{\theta, \omega}^2 + C_{\theta, \omega}^TC_{\theta, \omega}
		\end{bmatrix} \enspace \nonumber
\end{align}
where $\lambda_{\max}(M), \lambda_{\min}(M)$ denote the largest and smallest eigenvalue of matrix $M$.

\begin{thm}[Exponential Convergence: SGA]
	\label{thm:stable-sga}
	Consider $U(\theta, \omega): \mathbb{R}^p \times \mathbb{R}^q \rightarrow \mathbb{R}$ that satisfies Assumption~\ref{asmp:convex-concave} for some radius $r>0$ near a stable local Nash equilibrium $(\theta_*, \omega_*)$ as in Definition~\ref{def:slne}. 
	Suppose the initialization satisfies $(\theta_0, \omega_0) \in B_2((\theta_*, \omega_*), r)$.  Then the SGA dynamics \eqref{e:sga} with fixed learning rate $$\eta = \sqrt{\alpha} / \beta \enspace ,$$ ($\alpha, \beta$ defined in Eqn.~\eqref{eq:alpha_beta})
	obtains an $\epsilon$-minimizer such that $(\theta_T, \omega_T) \in B_2((\theta_*, \omega_*), \epsilon)$, as long as 
	\begin{align*}
		T \geq T_{\rm SGA} := \left\lceil 2\frac{\beta}{\alpha} \log \frac{r}{\epsilon} \right\rceil \enspace .
	\end{align*}
\end{thm}

\begin{remark}\label{remark:GD}
	\rm 
	It is interesting to compare the convergence speed of the saddle point dynamics to conventional gradient descent in one variable, for a strongly-convex function. We remind the reader that to obtain an $\epsilon$-minimizer for a strongly-convex function, one needs the following number of iterations of gradient descent,
		\begin{align*}
			T_{\rm GD} := \max\left\{ \frac{\lambda_{\max} \left(A_{\theta, \omega} \right)}{\lambda_{\min} \left(A_{\theta, \omega} \right)} \log \frac{r}{\epsilon}, ~~ \frac{\lambda_{\max} \left(B_{\theta, \omega} \right)}{\lambda_{\min} \left(B_{\theta, \omega} \right)} \log \frac{r}{\epsilon} \right\} \enspace ,
		\end{align*}
		depending on whether we are optimizing with respect to $\theta$ or $\omega$, respectively.
	It is now evident that due to the presence of $C_{\theta, \omega}C_{\theta, \omega}^T$, the convergence of two-player SGA to a saddle-point can be significantly slower than convergence of single-player gradient descent. In particular, applying the eigenvalue interlacing theorem to the principal submatrix $A_{\theta, \omega}^2 + C_{\theta, \omega}C_{\theta, \omega}^T$ of $F_{\theta,\omega}$ we obtain
		\begin{align*}
		\lambda_{\max}\left( F_{\theta,\omega} \right) & \geq \lambda_{\max}\left( A_{\theta, \omega}^2 + C_{\theta, \omega}C_{\theta, \omega}^T \right) \enspace , \\
		% & \geq  \lambda_{\max}\left( A_{\theta, \omega}^2  \right) + \lambda_{\rm min}(C_{\theta,\omega}C_{\theta,\omega}^T) \enspace , \\
		& \geq  \lambda_{\max}\left( A_{\theta, \omega}^2  \right) \enspace .
		\end{align*} Therefore the convergence of SGA is slower than that in the conventional GD\footnote{Recall that $\frac{\lambda_{\max}(A_{\theta,\omega}^2)}{\lambda_{\min}(A_{\theta,\omega}^2)} \geq \frac{\lambda_{\max}(A_{\theta,\omega})}{\lambda_{\min}(A_{\theta,\omega})}$.}
		% Thus, the convergence time of gradient descent is smaller than, \js{Come back to this.}
		\begin{align*}
			T_{\rm SGA} \geq T_{\rm GD}\enspace.
		\end{align*}
		We would like to emphasize that for the saddle point convergence, the slow-down effect of the interaction term $C_{\theta, \omega}$ is explicit in our non-asymptotic analysis. 
		
		The intuition that the discrete-time SGA dynamics cycles inward to a stable Nash equilibrium exponentially fast can be seen in the following way. The presence of the off-diagonal anti-symmetric component in Eqn.~\eqref{eq:abrev} means that the associated linear operator of the discrete-time dynamics has complex eigenvalues, which results in periodic cycling behavior. However, due to the explicit choice of $\eta$, the distance to stable Nash equilibrium is shrinking exponentially fast. The local exponential stability in the infinitesimal/asymptotic case when $\eta \rightarrow 0$ has already been studied in a paper \cite{nagarajan2017gradient} (Theorem 3.1 therein) by showing the Jacobian matrix of a particular form of GAN objective is Hurwitz (has all strictly negative eigenvalues). There are two distinct differences in our result: (1) we provide non-asymptotic convergence, with specific guidance on the choice of learning rate $\eta$; (2) our analysis goes through analyzing the singular values (which is rather different from the modulus of eigenvalue for a general matrix), instead of involving the complex eigenvalues, and this simple technique generalizes to four other modified saddle point dynamics which we discuss in the next section.  
\end{remark}

In fact, one can show that the slow-down effect of the interaction term $C_{\theta, \omega}$ for SGA in the above theorem is indeed necessary. 
\begin{coro}[Simple Lower Bound for SGA]
	\label{cor:lower-bound}
	Consider $U(\theta, \omega) = \frac{1}{2} \theta^T \theta - \frac{1}{2} \omega^T \omega + \theta^T C \omega$ with $p=q$ and $C \in \mathbb{R}^{p \times q}$ full rank. Then the SGA dynamics \eqref{e:sga} with any fixed learning rate $\eta$ satisfies
	\begin{align*}
		\| \theta_{t+1} \|^2 +  \| \omega_{t+1} \|^2 \geq \frac{\lambda_{\min}(C^T C)}{1+ \lambda_{\min}(C^T C)} \left( \| (\theta_{t} \|^2 +  \| \omega_{t} \|^2 \right).
	\end{align*}
\end{coro}
The above corollary shows that for any stepsize $\eta$, to obtain $\epsilon$-solution, the number of SGA iteration is at least $\Omega((1+\lambda_{\min}(C^T C)) \log (1/\epsilon))$. Namely, the interaction term is indeed a curse to the convergence rate. 

\section{Unstable Case: Local Bi-Linear Problem}\label{sec:bilinear}

Oscillation and instability for SGA occurs when the problem is non-strongly convex-concave, as in the bi-linear game (or more precisely, at least linear in one player). This observation was first pointed out using a very simple linear game $U(x, y) = xy$ in \cite{salimans2016improved}.
More generally, as a result of Theorem~\ref{thm:stable-sga}, this phenomenon occurs when the local Nash equilibrium is non-stable, 
\begin{align*}
			&  \diag(
			A_{\theta_*, \omega_*} , B_{\theta_*, \omega_*}) \approx \mathbf{0} \iff \\
			& \begin{bmatrix}
			A_{\theta_*, \omega_*} & C_{\theta_*, \omega_*} \\
			-C_{\theta_*, \omega_*}^T & B_{\theta_*, \omega_*}
		\end{bmatrix} \approx \begin{bmatrix}
			\mathbf{0} & C_{\theta_*, \omega_*} \\
			-C_{\theta_*, \omega_*}^T & \mathbf{0}
		\end{bmatrix}\enspace.
\end{align*}

Let's consider an extreme case when 
$A_{\theta_*, \omega_*} = \mathbf{0}$ and $B_{\theta_*, \omega_*} = \mathbf{0}$.
In this case, we will use a novel unified non-asymptotic analysis to show that the following proposed dynamics can fix the oscillation problem and provide exponential convergence to unstable Nash equilibria:
\begin{itemize}
	\item[(1)] \textit{Optimistic Mirror Descent} (OMD) in \cite{daskalakis2017training}
	\item[(2)] A modified version of \textit{Predictive Methods} (PM) motivated from \cite{yadav2017stabilizing}
	\item[(3)] \textit{Implicit Updates}
	\item[(4)] \textit{Consensus Optimization} (CO) introduced in \cite{mescheder2017numerics}
\end{itemize}
Our analysis shows that these stabilizing techniques, at a high level, all manipulate the dynamics to utilize the curvature generated by the interaction term $C_{\theta_*, \omega_*} C_{\theta_*, \omega_*}^T$\,---\, which we refer to as the ``blessing'' of the interaction term, to contrast with the ``slow-down effect'' of the interaction term in the strongly convex-concave case (Theorem \ref{thm:stable-sga}). Once again, as alluded to in the introduction, this fast linear-rate convergence result in the non-strongly convex-concave two-player game should be contrasted with the significantly slower sub-linear convergence rate for all first-order-methods in convex but non-strongly convex single-player optimization. The latter was proved by a lower bound argument in \cite[Theorem 2.1.7]{nesterov2013introductory}. The main result proved in this section is informally stated
\begin{thm}[Informal: Unstable Case]
	All these four modified dynamics, in the bi-linear game, enjoy the last iterate exponential convergence guarantee. 
\end{thm}

% We will start with the minimal unstable bilinear game, explaining why oscillation can happen when the game is non-strongly convex-concave, and to reveal and understand the new features of the modified dynamics.
The bilinear game can be motivated by considering the Taylor expansion of a general smooth two-player game around a non-stable Nash equilibrium ($A, B \approx \mathbf{0}$), 
% \begin{equation*}
% 	U(\theta, \omega) = \frac{1}{2} \theta^T A \theta + \frac{1}{2} \omega^T B \omega + \theta^T C \omega + \cdots \approx \theta^T C \omega \enspace ,
% \end{equation*}
assuming that $(\theta_\ast, \omega_\ast) = (\mathbf{0},\mathbf{0})$. Now consider the simple bi-linear game
$U(\theta,  \omega) = \theta^T C \omega$.
With the SGA dynamics defined in \eqref{e:sga}, one can easily verify that 
\begin{align*}
	% &\begin{bmatrix}
	% 	\theta_{t+1} \\
	% 	\omega_{t+1}
	% \end{bmatrix} = \left( I -  \eta \begin{bmatrix}
	% 	0& C \\
	% 	-C^T & 0
	% \end{bmatrix}  \right) \cdot \begin{bmatrix}
	% 	 	\theta_t  \\
	% 		\omega_t
	% 		\end{bmatrix} \\
	&\| \theta_{t+1} \|^2 \geq \left(1 + \eta^2 \lambda_{\min} (CC^T) \right) \| \theta_{t} \|^2 \enspace , \\
	& \| \omega_{t+1} \|^2 \geq \left(1 + \eta^2 \lambda_{\min} (C^TC) \right) \| \omega_{t} \|^2 \enspace.
\end{align*}
Therefore, the continuous limit $\eta \rightarrow 0$ is cycling around a sphere, while with any practical learning rate $\eta \neq 0$, the distance to the Nash equilibrium can be increasing exponentially instead of converging. 
Per Theorem~\ref{thm:stable-sga} and the discussion above, instability for SGA only occurs when the local game is approximately bi-linear.
From now on, therefore, we will focus on the simplest unstable form of the game, the bi-linear game, to isolate the main idea behind fixing the instability problem. The proof technique can be extended to more general settings, with a sacrifice of simplicity.

\subsection{(Improved) Optimistic Mirror Descent}

\citet*{daskalakis2017training} employed Optimistic Mirror Descent (OMD) \citep{rakhlin2013optimization} motivated by online learning to solve the instability problem in GANs. Here we provide a stronger result, showing that the last iterate of OMD enjoys \textit{exponential convergence} for bi-linear games. We note that although the last-iterate convergence of this OMD procedure was already rigorously proved in \cite{daskalakis2017training}, the exponential convergence is not known to the best of our knowledge. 
	
\begin{thm}[Exponential Convergence: OMD]
	\label{lem:unstable-omd}
	Consider a bi-linear game $U(\theta,  \omega) = \theta^T C \omega.$ Assume $p = q$ and $C$ is full rank. Then the OMD dynamics,
	\begin{align}
		\label{eq:omd}
		\theta_{t+1} &= \theta_{t} - 2\eta \nabla_{\theta} U(\theta_{t}, \omega_{t}) + \eta \nabla_{\theta} U(\theta_{t-1}, \omega_{t-1})\enspace, \nonumber \\
		\omega_{t+1} &= \omega_{t} + 2\eta \nabla_{\omega} U(\theta_{t}, \omega_{t}) - \eta \nabla_{\omega} U(\theta_{t-1}, \omega_{t-1})\enspace,
	\end{align}  
	with the learning rate
	\begin{align*}
		\eta = \frac{1}{2\sqrt{2\lambda_{\max}(CC^T)}}\enspace,
	\end{align*}
	obtains an $\epsilon$-minimizer such that $(\theta_T, \omega_T) \in B_2(\epsilon)$, provided
	\begin{align*}
		T \geq T_{\rm OMD} := \left\lceil 16\frac{\lambda_{\max}(CC^T)}{\lambda_{\min}(CC^T)} \log \frac{4\sqrt{2}r}{\epsilon} \right\rceil \enspace,
	\end{align*}
	under the assumption that $\| (\theta_0, \omega_0) \|, \|(\theta_1, \omega_1) \|\leq r$.
\end{thm}

Let us compare our result with the last-iterate convergence result in \cite{daskalakis2017training}. Roughly speaking, \cite[Theorem 1]{daskalakis2017training} asserts that to obtain an $\epsilon$-minimizer, one requires a learning rate scaling as $\eta(\epsilon) \asymp \epsilon^{2}$ and a number of iterations bounded by $$T \succsim \epsilon^{-4} \log \frac{1}{\epsilon} \cdot  {\rm Poly}\left(\frac{\lambda_{\max}(CC^T)}{\lambda_{\min}(CC^T)}\right) \enspace .$$
In contrast, we show that with step size $\eta$ chosen independently of $\epsilon$, the last iterate of OMD falls within $\epsilon$ of the saddle point after a number of iterations given by $$T \succsim \log \frac{1}{\epsilon} \cdot \frac{\lambda_{\max}(CC^T)}{\lambda_{\min}(CC^T)} \enspace .$$ In other words, we improved the dependence on $\epsilon$ from polynomial to logarithmic. This improved analysis also coincides with the exponential convergence found in simulations.

\subsection{(Modified) Predictive Methods}

From a very different motivation in ordinary differential equations, \cite{yadav2017stabilizing} proposed Predictive Methods (PM) to fix the instability problem. The intuition is to evaluate the gradient at a predictive future location and then perform the update. In this section, we propose and analyze a modified version of the predictive method (for simultaneous gradient updates), inspired by \cite{yadav2017stabilizing}.

Consider the following modified PM dynamics,
\begin{align}
	\label{eq:pm}
	\text{predictive step:} \quad \theta_{t+1/2} &= \theta_t - \gamma \nabla_{\theta} U(\theta_t, \omega_t) \enspace, \nonumber \\
	 \quad \omega_{t+1/2} &= \omega_t + \gamma \nabla_{\omega} U(\theta_t, \omega_t) \enspace; \nonumber \\
	\text{gradient step:} \quad \theta_{t+1} &= \theta_{t} - \eta \nabla_{\theta} U(\theta_{t+1/2}, \omega_{t+1/2}) \enspace, \nonumber \\
	 \quad \omega_{t+1} &= \omega_{t} + \eta \nabla_{\omega} U(\theta_{t+1/2}, \omega_{t+1/2}) \enspace.
\end{align}

\begin{thm}[Exponential Convergence: PM]
	\label{lem:unstable-pm}
	Consider a bi-linear game $U(\theta,  \omega) = \theta^T C \omega.$ Assume $p = q$ and $C$ is full rank. Fix some $\gamma > 0$. Then the PM dynamics in Eqn.~\eqref{eq:pm} 
	with learning rate
	\begin{align*}
		\eta = \frac{\gamma \lambda_{\min}(CC^T)}{\lambda_{\max}(CC^T) + \gamma^2 \lambda^2_{\max}(CC^T)} \enspace, 
	\end{align*}
	obtains an $\epsilon$-minimizer such that $(\theta_T, \omega_T) \in B_2(\epsilon)$, provided
	\begin{align*}
		T \geq T_{\rm PM} := \left\lceil 2\frac{\gamma^2 \lambda^2_{\max}(CC^T) + \lambda_{\max}(CC^T)}{\gamma^2 \lambda^2_{\min}(CC^T)}  \log \frac{r}{\epsilon} \right\rceil \enspace,
	\end{align*}
	under the assumption that $\| (\theta_0, \omega_0) \|\leq r$.
\end{thm}

\subsection{Implicit Updates}
Implicit Update (IU) rules have been shown to be more robust compared to explicit updates, and typically match the performance or even outperform the latter empirically in online learning \citep{kulis2010implicit}. We will show that a simple adaptation of implicit updates for simultaneous gradient ascent/descent resolves the instability problem in the bi-linear case.

\begin{thm}[Exponential Convergence: IU]
	\label{lem:unstable-implicit}
	Consider a bi-linear game $U(\theta,  \omega) = \theta^T C \omega.$ Assume $p = q$ and $C$ is full rank. Then the implicit updates
	\begin{align*}
		\theta_{t+1} &= \theta_t - \eta \nabla_\theta U(\theta_{t+1}, \omega_{t+1})  \enspace,\\
		\omega_{t+1} &= \omega_t + \eta \nabla_\omega U(\theta_{t+1}, \omega_{t+1}) \enspace,
	\end{align*}
	with the learning rate 
	\begin{align*}
		\eta =  \frac{1}{\sqrt{\lambda_{\max}(CC^T)}}
	\end{align*}
	obtains an $\epsilon$-minimizer such that $(\theta_T, \omega_T) \in B_2(\epsilon)$, provided
	\begin{align*}
		T \geq T_{\rm IU} :=  \left\lceil (2+\sqrt{2})\frac{\lambda_{\max}(CC^T)}{\lambda_{\min}(CC^T)}  \log \frac{r}{\epsilon} \right\rceil
	\end{align*}
	under the assumption that $\| (\theta_0, \omega_0) \|\leq r$.
\end{thm}

\subsection{Consensus Optimization}
Consensus Optimization (CO) is another elegant attempt to fix the aforementioned problem, proposed in \cite{mescheder2017numerics}. The idea is to add a potential component to the pure-curl vector field associated with SGA in the bi-linear game, in order to attract the dynamics to the critical points. \citep{mescheder2017numerics, nagarajan2017gradient} analyzed the infinitesimal flow version of the consensus optimization, and intuitively showed that it pushes the real part of the eigenvalue away from $1$, to ensure asymptotic convergence. In this section, we provide a simple convergence analysis of the discretized dynamics, of the same flavor as the previous section. An upshot of the analysis is that it sheds light on possible choices of learning rate. 

	Recall that the regularization term defining consensus optimization is given by,
	\begin{align}
		\label{eq:consensus-field}
		R(\theta,  \omega) = \frac{1}{2}\left(\| \nabla_{\theta} U(\theta, \omega) \|^2 + \| \nabla_{\omega} U(\theta, \omega) \|^2 \right) \enspace.
	\end{align}
	Surprisingly, we find that the consensus optimization coincides with the modified predictive method for the bi-linear game, as described by the following
	
	\begin{thm}[Exponential Convergence: CO]
		\label{lem:unstable-co}
		Consider a bi-linear game $U(\theta,  \omega) = \theta^T C \omega.$ Assume $p = q$ and $C$ is full rank. Recall $R(\theta, \omega)$ defined in Eqn.~\eqref{eq:consensus-field}, and fix some $\gamma > 0$. 
		Then the CO dynamics with the same learning rate $\eta$ as in Thm.~\ref{lem:unstable-pm},
	 	\begin{align}
			\label{eq:co}
	 		\theta_{t+1} &= \theta_{t} - \eta \left[ \nabla_{\theta} U(\theta_t, \omega_t) + \gamma \nabla_{\theta} R(\theta_t,  \omega_t) \right] \enspace, \nonumber\\
	 		\omega_{t+1} &= \omega_{t} +  \eta \left[ \nabla_{\omega} U(\theta_t, \omega_t) - \gamma \nabla_{\omega} R(\theta_t,  \omega_t) \right] \enspace,
	 	\end{align}
		converges exponentially fast in the same way as the PM dynamics in Thm.~\ref{lem:unstable-pm}.    
	\end{thm}

\section{Experiments}\label{sec:experiments}
In the simplistic setting of bilinear games we have seen that exponential convergence can be achieved for appropriate choice of learning rate and this is indeed confirmed by numerical experiments as shown in Fig. \ref{fig:bilinear}. In reality, however, the assumption of bilinearity is not applicable to value functions of GAN form and indeed
\begin{figure}[h]
\centering
\includegraphics[width=\linewidth]{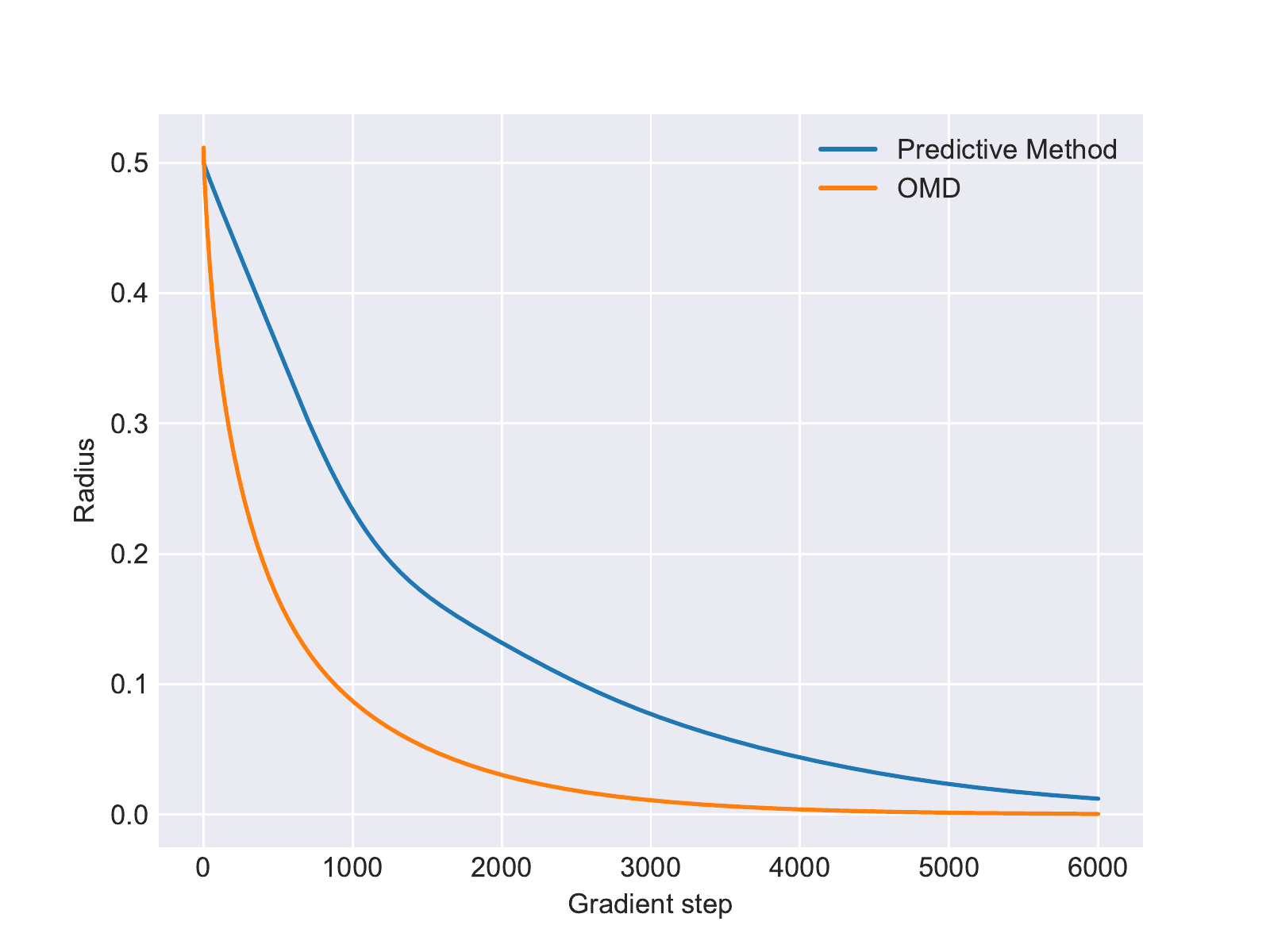}
\caption{Distance to Nash equilibrium as a function of gradient iteration for the bilinear game assuming $p=q=5$, $\gamma = 1$ and $r=0.5$. The components of the interaction matrix $C$ were chosen i.i.d. uniform on $[0,1]$.}
\label{fig:bilinear}
\end{figure}
recent large-scale studies of GAN optimization \citep{lucic2017gans} suggest that improvements from algorithmic changes mostly disappeared after taking into account hyper-parameter tuning and randomness of initialization. They conclude that ``future GAN research should be based on more systematic and objective evaluation procedures.'' Inspired by this conclusion, we conduct a systematic evaluation of the proposed optimization algorithms on two basic density learning problems, and introduce corresponding objective evaluation metrics. The goal of this analysis is not to achieve state-of-art performance, but rather to compare and contrast the existing proposals in a carefully controlled learning environment. We focus on the Wasserstein GAN formulation so that the value function is given by
 \begin{equation*}
	 U(\theta, \omega) = \underset{X \sim \mathcal{P}_{\rm real}}{\mathbb{E}} f_\omega(X) - \underset{Z \sim N(0, I_k)}{\mathbb{E}} f_\omega(g_\theta(Z)) \enspace ,
\end{equation*}
where $f_\omega: \mathbb{R}^d \to \mathbb{R}$ is a multi-layer neural network with $L$ hidden layers and rectifier non-linearities and the input distribution $\mathcal{P}_{\rm input}$ was chosen to be $k$-dimensional standard Gaussian noise. Following \cite{gulrajani2017improved}, we impose the Lipschitz-1 constraint on the discriminator network using the two-sided gradient penalty term $\Lambda(\omega)$ introduced in \cite[Eqn.~(3)]{gulrajani2017improved}. The consensus optimization loss is defined with respect to the value function as in \eqref{eq:consensus-field} without including gradient penalty. The combined loss function of the discriminator and generator are respectively,% \footnote{For discriminator functions as in \cite{liang2017fisher} (and $\gamma = \lambda = 0$) it holds that $\frac{d}{dt}\Vert \theta_t \Vert^2_2 = 2(L+1)L(\theta_t,\omega)$.}  
\begin{align*}
	L_{\rm dis}(\theta,\omega)
		& = -U(\theta,\omega) + \gamma \, R(\theta, \omega) + \lambda \, \Lambda(\omega) \enspace , \\
	L_{\rm gen}(\theta,\omega)
		& = U(\theta,\omega) + \gamma \, R(\theta, \omega) \enspace .
\end{align*}
The coefficients of the gradient penalty and consensus optimization terms were determined by a coarse parameter search and then locked to $\lambda = \gamma = 1$ throughout. In order to make close contact with our theoretical formalism, we optimize the above loss functions using simultaneous gradient updates with fixed learning rate of $\eta = 10^{-3}$. 
%To ensure reproducibility, all algorithms were independently implemented on top of the TFGAN and Keras framework.
\subsection{Learning Covariance of Multivariate Gaussian}
Consider the problem of learning the covariance matrix $\Sigma \in S^d_{++}$ of a $d$-dimensional multivariate Gaussian distribution $\mathcal{P}_{\rm real} = N(0, \Sigma)$ with non-degenerate covariance $\Sigma \succ 0$. Note that the learning problem is well-specified if we choose the generator function  $g_\theta : \mathbb{R}^k \to \mathbb{R}^d$  to be a simple linear transformation of the $k$-dimensional latent space ($k \geq d$). Although the GAN approach is clearly overkill for this simple density estimation problem, we find this example illuminating because it affords some analytical tractability for the otherwise intractable general GAN value function \eqref{e:gan}. Specifically, if we choose the discriminator function $f_\omega : \mathbb{R}^d \to \mathbb{R}$ to be a neural network with $L=1$ hidden layer consisting of $H$ hidden units with rectifier nonlinearities, and set biases to zero, then the explicit functional forms of discriminator and generator are respectively,
\begin{align*}
f_\omega (x) = \sum_{i=1}^H v_i \langle w_i , x \rangle \, \mathbf{1}_{\{ \langle w_i,  x \rangle \geq 0 \}} \enspace , \quad \quad g_\theta(z)  = V z \enspace ,
\end{align*}
where $\omega \in \{ w_i \in \mathbb{R}^d, v_i \in \mathbb{R} : \forall i \in [H] \}$ and $\theta \in \{ V \in \mathbb{R}^{d\times k} \}$ are the discriminator and generator parameters, respectively. 
If, moreover, we express the covariance matrix as $\Sigma = AA^T$, then the value function can be expressed in closed form as,
\begin{equation}\label{e:toy}
U(\theta, \omega) = \mathrm{const} \times \sum_{i=1}^H v_i \big[ \Vert A^T w_i \Vert - \Vert V^T w_i \Vert \big] \enspace .
\end{equation}
The above analytical form of the value function sheds some light on the nature of the local Nash equilibrium solution concept. In particular, if one solves for the condition of being a Nash equilibrium, one does not conclude that $V_\ast V_\ast^T = \Sigma$. The result depends on the rank of the matrix $[w_1,\ldots,w_H]$.

The evaluation of different optimization algorithms involved comparing the target density $\mathcal{P}_{\rm real}=N(0,\Sigma)$ and the analytical generator density $\mathcal{P}_{\rm fake}=N(0, VV^T)$ after $t=10^5$ training iterations (Fig.~\ref{fig:boxplot}). For simplicity, we chose the evaluation metric to be the Frobenius norm of the difference between the covariance matrices $\Vert \Sigma - VV^T \Vert_{\rm F}$.
The covariance learning experiments were conducted in the well-specified and over-parametrized regime ($k=16$, $d=2$) using $H = 128$ hidden units for the discriminator network. 
%We also performed a head-to-head comparison of simultaneous and alternating update methods on this distribution and found negligible difference (see Fig.~\ref{fig:altsim} of appendices).

\subsection{Mixture of Gaussians}
In practical applications, GANs are typically trained using the empirical distribution of the samples, where the samples are drawn from an idealized multi-modal probability distribution.  
To capture the notion of a multi-modal data distribution, we focus on a mixture of 8 Gaussians with means located at the vertices of a regular octagon inscribed in the unit circle, where each component has a fixed diagonal covariance of width $\sigma = 0.03$. In contrast to previous visual-based evaluations, we estimate the Wasserstein-1 distance $W_1(\mathcal{P}_{\rm real}, \mathcal{P}_{\rm fake})$ between the target density $\mathcal{P}_{\rm real}$ and the distribution $\mathcal{P}_{\rm fake}$ of the random variable $g_\theta(Z)$ implied by the trained generator network. The estimate is obtained by solving the linear program which computes the earth mover's distance between the sample estimates $\widehat{\mathcal{P}}_{\rm real} = \frac{1}{m}\sum_{i=1}^m \delta_{X_i}$ and $\widehat{\mathcal{P}}_{\rm fake} = \frac{1}{m}\sum_{i=1}^m \delta_{g_\theta(Z_i)}$, respectively, and approaches the population version $W_1(\mathcal{P}_{\rm real}, \mathcal{P}_{\rm fake})$ as the number of samples $m \to \infty$.

The experiments with the mixture of Gaussians used 2 dimensional Gaussian as input ($k=2$). Both the generator and discriminator networks consisted of $L=4$ hidden layers with $H=128$ units per hidden layer. The estimate of the Wasserstein-1 distance was calculated using a sample size of $m=512$ after training for $t = 5\cdot10^4$ iterations. It is clear from Fig.~\ref{fig:density} that the Wasserstein-1 distance correlates closely with the visual fit to the target distribution. The empirical evaluation (Fig.~\ref{fig:boxplot}) shows that the separation between consensus optimization and competing algorithms disappears on the mixture distribution, suggesting that the qualitative ranking is not robust to the choice of loss landscape. These findings demand deeper understanding of the global structure of the landscape, including the formulation of regularization to tame the notoriously difficult GAN optimization \citep{arbel2018gradient}, which is not captured by our local stability analysis.

\begin{figure}[h]
\centering
\includegraphics[width=\linewidth]{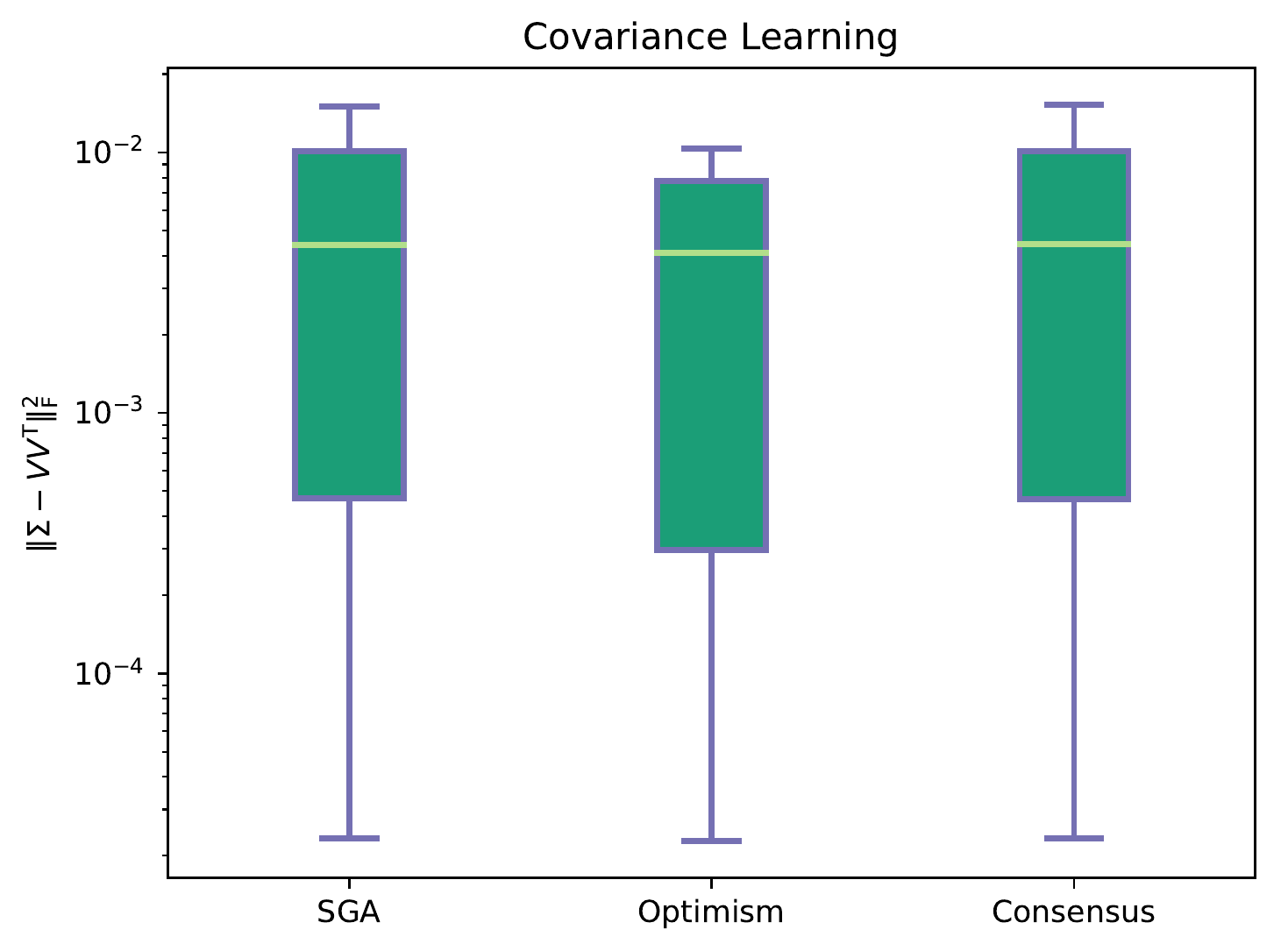}
\includegraphics[width=\linewidth]{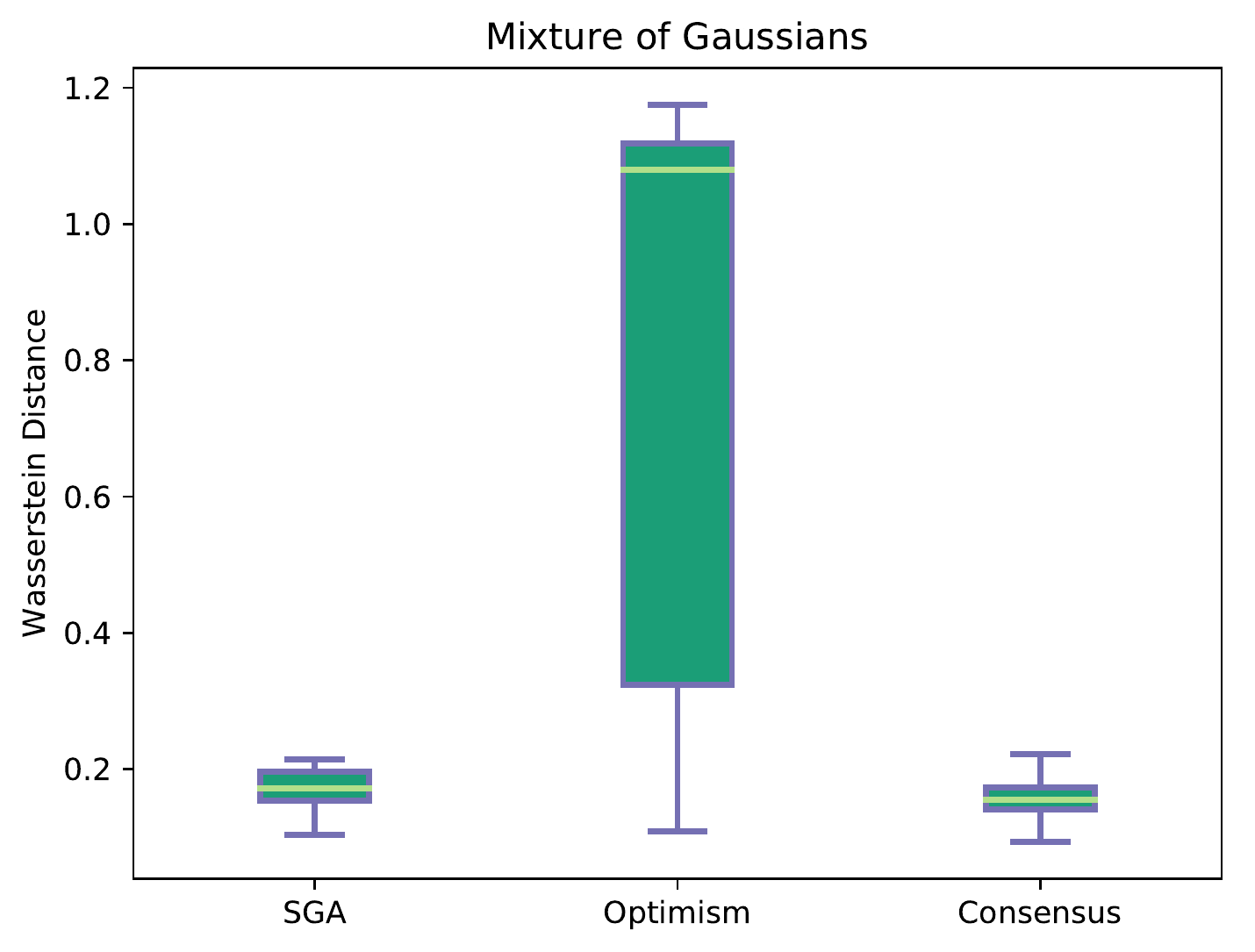}
\caption{Evaluation metrics for covariance learning (top) and mixture of Gaussians learning (bottom) using different dynamical systems after $t=10^5$ and $t=5\cdot 10^4$ training iterations, respectively and 16 random seeds. Note that for covariance learning, we use the log-scale on $y$-axis.}
\label{fig:boxplot}
\end{figure}

\begin{figure}[h]
\centering
\includegraphics[width=0.6\linewidth]{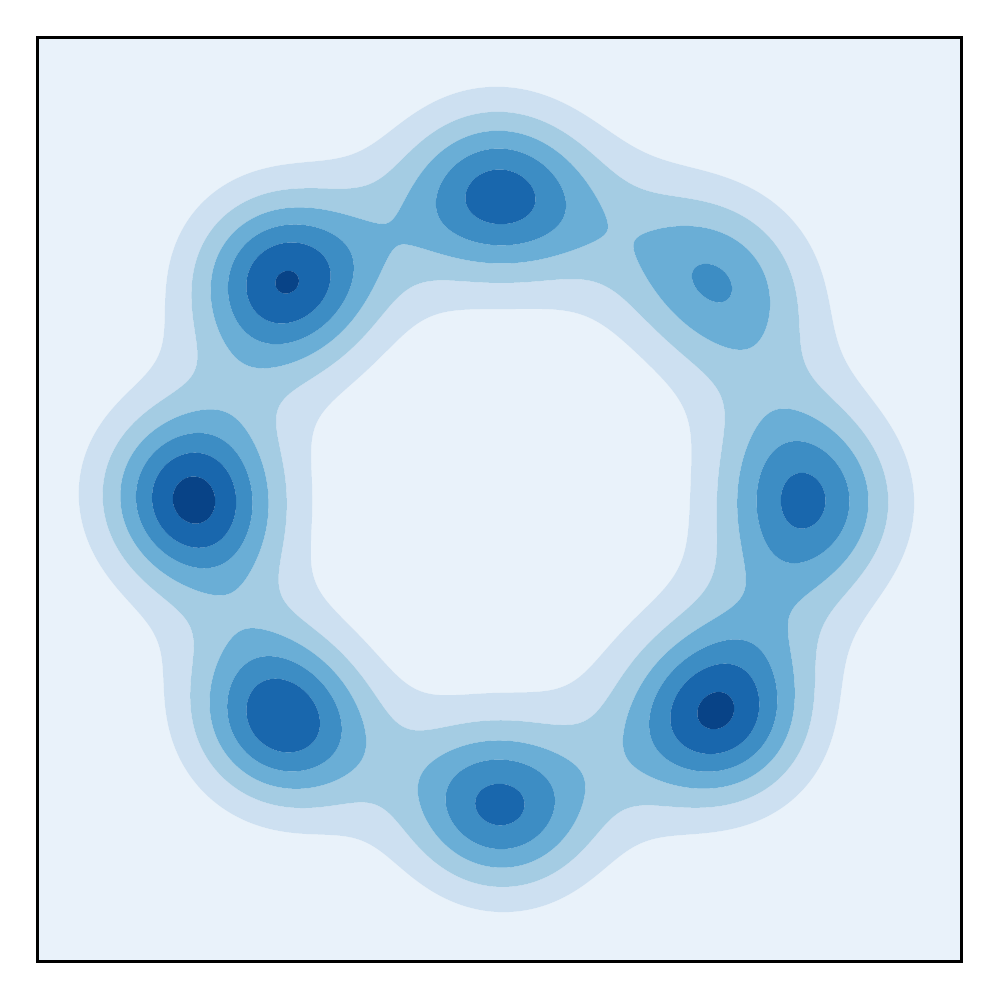}
\includegraphics[width=0.6\linewidth]{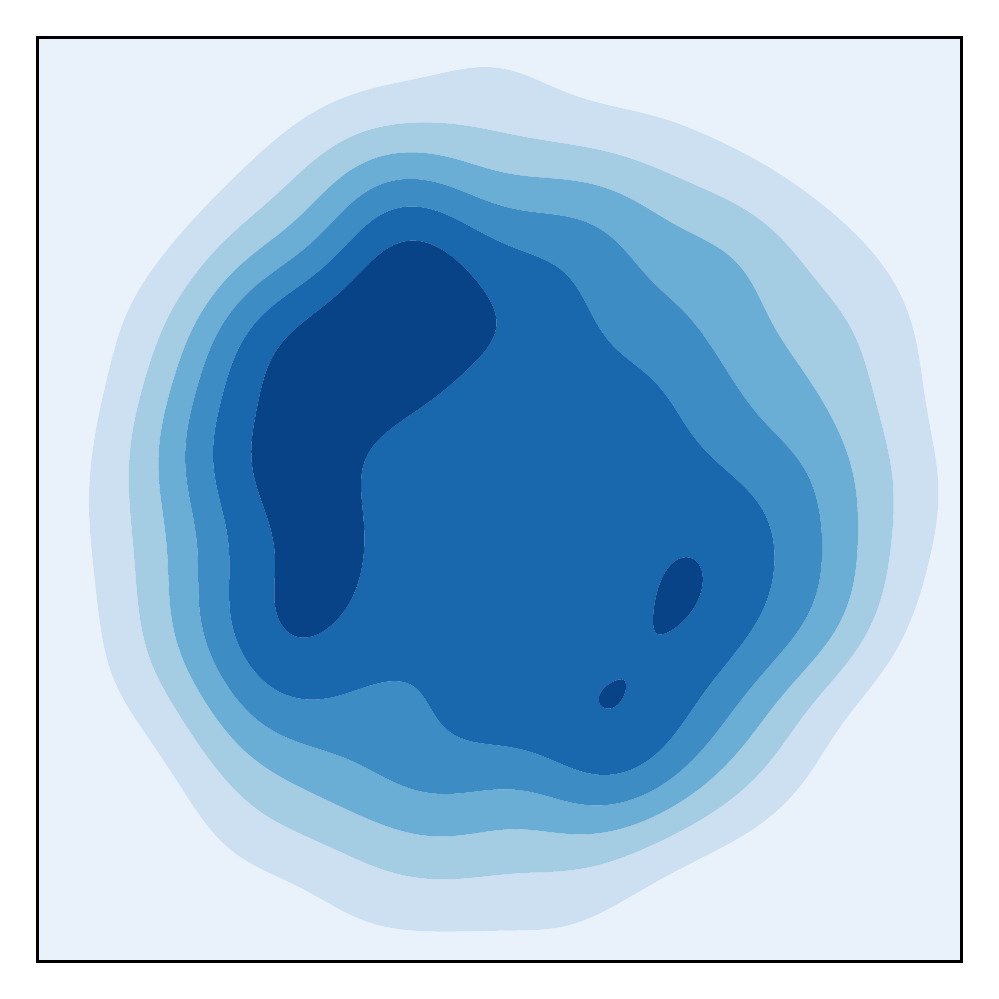}
\caption{Density plots of best and worst generator distribution measured by empirical Wasserstein-1 distance from the target distribution, across all baselines amongst 16 random seeds (excluding non-convergent runs) after training for $5\times 10^4$ iterations. Top: Consensus ($W_1=0.093$). Bottom: OMD ($W_1=0.367$).}
\label{fig:density}
\end{figure}

\section{Conclusions and Future Work}
In this paper we made a first step towards understanding the local convergence rate of the discrete-time gradient-based saddle point dynamics for solving smooth two-player zero-sum games, including GANs as the leading motivation. The focus of the paper is on illustrating how local geometry affects the convergence speed and choice of learning-rate for both stable and unstable local Nash Equilibria. A curious fact we proved is that modified first-order dynamics such as OMD converge with linear-rate to unstable local Nash Equilibria, as a consequence of the interaction term between the two players.

We acknowledge that there are still critical steps left open by our analysis in order to understand the effectiveness of heuristic methods for training GANs for distribution learning. Solving this problem requires an understanding of the ability of various stable/unstable local Nash Equilibria to represent distributions in the statistical sense. To the best of our knowledge, it remains unclear whether convergence to the stable local solution concept (Definition \ref{def:slne}) is better than converging/oscillating/escaping an unstable local solution, in terms of distribution learning. Recent progress by \cite{daskalakis2018limit, daskalakis2018last} employs the stable manifold Theorem \citep{galor2007discrete} to show that certain dynamics avoid unstable local solutions (barring initialization in a set of Lebesgue measure zero). Overall, a satisfactory theory\,---\,in both the computational and statistical sense\,---\,for answering how heuristic gradient-based saddle point dynamics for GANs are able to learn distributions is still wide open for future investigation.

\bibliography{ref}

\begin{thebibliography}{28}
\providecommand{\natexlab}[1]{#1}
\providecommand{\url}[1]{\texttt{#1}}
\expandafter\ifx\csname urlstyle\endcsname\relax
  \providecommand{\doi}[1]{doi: #1}\else
  \providecommand{\doi}{doi: \begingroup \urlstyle{rm}\Url}\fi

\bibitem[Arbel et~al.(2018)Arbel, Sutherland, Bi{\'n}kowski, and
  Gretton]{arbel2018gradient}
Michael Arbel, Dougal~J Sutherland, Miko{\l}aj Bi{\'n}kowski, and Arthur
  Gretton.
\newblock On gradient regularizers for mmd gans.
\newblock \emph{arXiv preprint arXiv:1805.11565}, 2018.

\bibitem[Arjovsky et~al.(2017)Arjovsky, Chintala, and
  Bottou]{arjovsky2017wasserstein}
Martin Arjovsky, Soumith Chintala, and L{\'e}on Bottou.
\newblock Wasserstein gan.
\newblock \emph{arXiv preprint arXiv:1701.07875}, 2017.

\bibitem[Arora and Zhang(2017)]{arora2017gans}
Sanjeev Arora and Yi~Zhang.
\newblock Do gans actually learn the distribution? an empirical study.
\newblock \emph{arXiv preprint arXiv:1706.08224}, 2017.

\bibitem[Arora et~al.(2017)Arora, Risteski, and Zhang]{arora2017theoretical}
Sanjeev Arora, Andrej Risteski, and Yi~Zhang.
\newblock Theoretical limitations of encoder-decoder gan architectures.
\newblock \emph{arXiv preprint arXiv:1711.02651}, 2017.

\bibitem[Cherukuri et~al.(2017)Cherukuri, Gharesifard, and
  Cortes]{cherukuri2017saddle}
Ashish Cherukuri, Bahman Gharesifard, and Jorge Cortes.
\newblock Saddle-point dynamics: conditions for asymptotic stability of saddle
  points.
\newblock \emph{SIAM Journal on Control and Optimization}, 55\penalty0
  (1):\penalty0 486--511, 2017.

\bibitem[Daskalakis and Panageas(2018{\natexlab{a}})]{daskalakis2018last}
Constantinos Daskalakis and Ioannis Panageas.
\newblock Last-iterate convergence: Zero-sum games and constrained min-max
  optimization.
\newblock \emph{arXiv preprint arXiv:1807.04252}, 2018{\natexlab{a}}.

\bibitem[Daskalakis and Panageas(2018{\natexlab{b}})]{daskalakis2018limit}
Constantinos Daskalakis and Ioannis Panageas.
\newblock The limit points of (optimistic) gradient descent in min-max
  optimization.
\newblock \emph{arXiv preprint arXiv:1807.03907}, 2018{\natexlab{b}}.

\bibitem[Daskalakis et~al.(2017)Daskalakis, Ilyas, Syrgkanis, and
  Zeng]{daskalakis2017training}
Constantinos Daskalakis, Andrew Ilyas, Vasilis Syrgkanis, and Haoyang Zeng.
\newblock Training gans with optimism.
\newblock \emph{arXiv preprint arXiv:1711.00141}, 2017.

\bibitem[Galor(2007)]{galor2007discrete}
Oded Galor.
\newblock \emph{Discrete dynamical systems}.
\newblock Springer Science \& Business Media, 2007.

\bibitem[Goodfellow et~al.(2014)Goodfellow, Pouget-Abadie, Mirza, Xu,
  Warde-Farley, Ozair, Courville, and Bengio]{goodfellow2014generative}
Ian Goodfellow, Jean Pouget-Abadie, Mehdi Mirza, Bing Xu, David Warde-Farley,
  Sherjil Ozair, Aaron Courville, and Yoshua Bengio.
\newblock Generative adversarial nets.
\newblock In \emph{Advances in neural information processing systems}, pages
  2672--2680, 2014.

\bibitem[Gulrajani et~al.(2017)Gulrajani, Ahmed, Arjovsky, Dumoulin, and
  Courville]{gulrajani2017improved}
Ishaan Gulrajani, Faruk Ahmed, Martin Arjovsky, Vincent Dumoulin, and Aaron~C
  Courville.
\newblock Improved training of wasserstein gans.
\newblock In \emph{Advances in Neural Information Processing Systems}, pages
  5769--5779, 2017.

\bibitem[Heusel et~al.(2017)Heusel, Ramsauer, Unterthiner, Nessler, and
  Hochreiter]{heusel2017gans}
Martin Heusel, Hubert Ramsauer, Thomas Unterthiner, Bernhard Nessler, and Sepp
  Hochreiter.
\newblock Gans trained by a two time-scale update rule converge to a local nash
  equilibrium.
\newblock In \emph{Advances in Neural Information Processing Systems}, pages
  6629--6640, 2017.

\bibitem[Kulis and Bartlett(2010)]{kulis2010implicit}
Brian Kulis and Peter~L Bartlett.
\newblock Implicit online learning.
\newblock In \emph{Proceedings of the 27th International Conference on Machine
  Learning (ICML-10)}, pages 575--582, 2010.

\bibitem[Liang(2017)]{liang2017well}
Tengyuan Liang.
\newblock How well can generative adversarial networks learn densities: A
  nonparametric view.
\newblock \emph{arXiv preprint arXiv:1712.08244}, 2017.

\bibitem[Liang(2018)]{liang2018well}
Tengyuan Liang.
\newblock On how well generative adversarial networks learn densities:
  Nonparametric and parametric results.
\newblock \emph{arXiv preprint arXiv:1811.03179}, 2018.

\bibitem[Liu and Chaudhuri(2018)]{liu2018inductive}
Shuang Liu and Kamalika Chaudhuri.
\newblock The inductive bias of restricted f-gans.
\newblock \emph{arXiv preprint arXiv:1809.04542}, 2018.

\bibitem[Liu et~al.(2017)Liu, Bousquet, and Chaudhuri]{liu2017approximation}
Shuang Liu, Olivier Bousquet, and Kamalika Chaudhuri.
\newblock Approximation and convergence properties of generative adversarial
  learning.
\newblock \emph{arXiv preprint arXiv:1705.08991}, 2017.

\bibitem[Lucic et~al.(2017)Lucic, Kurach, Michalski, Gelly, and
  Bousquet]{lucic2017gans}
Mario Lucic, Karol Kurach, Marcin Michalski, Sylvain Gelly, and Olivier
  Bousquet.
\newblock Are gans created equal? a large-scale study.
\newblock \emph{arXiv preprint arXiv:1711.10337}, 2017.

\bibitem[Mescheder et~al.(2017)Mescheder, Nowozin, and
  Geiger]{mescheder2017numerics}
Lars Mescheder, Sebastian Nowozin, and Andreas Geiger.
\newblock The numerics of gans.
\newblock \emph{arXiv preprint arXiv:1705.10461}, 2017.

\bibitem[Metz et~al.(2016)Metz, Poole, Pfau, and
  Sohl-Dickstein]{metz2016unrolled}
Luke Metz, Ben Poole, David Pfau, and Jascha Sohl-Dickstein.
\newblock Unrolled generative adversarial networks.
\newblock \emph{arXiv preprint arXiv:1611.02163}, 2016.

\bibitem[Nagarajan and Kolter(2017)]{nagarajan2017gradient}
Vaishnavh Nagarajan and J~Zico Kolter.
\newblock Gradient descent gan optimization is locally stable.
\newblock In \emph{Advances in Neural Information Processing Systems}, pages
  5591--5600, 2017.

\bibitem[Nesterov(2013)]{nesterov2013introductory}
Yurii Nesterov.
\newblock \emph{Introductory lectures on convex optimization: A basic course},
  volume~87.
\newblock Springer Science \& Business Media, 2013.

\bibitem[Nowozin et~al.(2016)Nowozin, Cseke, and Tomioka]{nowozin2016f}
Sebastian Nowozin, Botond Cseke, and Ryota Tomioka.
\newblock f-gan: Training generative neural samplers using variational
  divergence minimization.
\newblock In \emph{Advances in Neural Information Processing Systems}, pages
  271--279, 2016.

\bibitem[Pfau and Vinyals(2016)]{pfau2016connecting}
David Pfau and Oriol Vinyals.
\newblock Connecting generative adversarial networks and actor-critic methods.
\newblock \emph{arXiv preprint arXiv:1610.01945}, 2016.

\bibitem[Rakhlin and Sridharan(2013)]{rakhlin2013optimization}
Sasha Rakhlin and Karthik Sridharan.
\newblock Optimization, learning, and games with predictable sequences.
\newblock In \emph{Advances in Neural Information Processing Systems}, pages
  3066--3074, 2013.

\bibitem[Salimans et~al.(2016)Salimans, Goodfellow, Zaremba, Cheung, Radford,
  and Chen]{salimans2016improved}
Tim Salimans, Ian Goodfellow, Wojciech Zaremba, Vicki Cheung, Alec Radford, and
  Xi~Chen.
\newblock Improved techniques for training gans.
\newblock In \emph{Advances in Neural Information Processing Systems}, pages
  2234--2242, 2016.

\bibitem[Singh et~al.(2000)Singh, Kearns, and Mansour]{singh2000nash}
Satinder Singh, Michael Kearns, and Yishay Mansour.
\newblock Nash convergence of gradient dynamics in general-sum games.
\newblock In \emph{Proceedings of the Sixteenth conference on Uncertainty in
  artificial intelligence}, pages 541--548. Morgan Kaufmann Publishers Inc.,
  2000.

\bibitem[Yadav et~al.(2017)Yadav, Shah, Xu, Jacobs, and
  Goldstein]{yadav2017stabilizing}
Abhay Yadav, Sohil Shah, Zheng Xu, David Jacobs, and Tom Goldstein.
\newblock Stabilizing adversarial nets with prediction methods.
\newblock \emph{Accepted at ICLR 2018}, 2017.

\end{thebibliography}
\bibliographystyle{plainnat}

\newpage
\appendix
\onecolumn
\section{Technical Proofs}
\label{sec:Appendix}
\begin{proof}[Proof of Theorem~\ref{thm:stable-sga}]
	Define the line interpolation between two points, 
	\begin{align*}
		\theta(x) &= x\theta_t + (1-x)\theta_* \enspace, \\
		\omega(x) &= x\omega_t + (1 - x)\omega_* \enspace.
	\end{align*}
	Then the SGA dynamics can be written as (using Taylor's theorem with remainder)
	\begin{align*}
		\begin{bmatrix}
			\theta_{t+1} - \theta_*\\
			\omega_{t+1} - \omega_*
		\end{bmatrix} &=  \begin{bmatrix}
			\theta_t - \theta_*\\
			\omega_t - \omega_*
		\end{bmatrix} - \eta \begin{bmatrix}
			\nabla_{\theta} U(\theta_t, \omega_t) \\
			-\nabla_{\omega} U(\theta_t, \omega_t)
		\end{bmatrix} \enspace,
		 \\
		&= \begin{bmatrix}
			\theta_t - \theta_*\\
			\omega_{t} - \omega_*
		\end{bmatrix}
		 - \eta \int_{0}^1 
		\begin{bmatrix}
			\nabla_{\theta\theta} U( \theta(x), \omega(x)) & \nabla_{\theta \omega} U( \theta(x), \omega(x)) \\
			-\nabla_{\omega \theta} U( \theta(x), \omega(x)) & -\nabla_{\omega \omega} U( \theta(x), \omega(x)) 
		\end{bmatrix}    
			 d x \cdot \begin{bmatrix}
			 	\theta_t - \theta_* \\
				\omega_{t} - \omega_*
			 \end{bmatrix} \enspace, \\
		& =  \int_0^1 \left( I - \eta \begin{bmatrix}
			\nabla_{\theta\theta} U( \theta(x), \omega(x)) & \nabla_{\theta \omega} U( \theta(x), \omega(x)) \\
			-\nabla_{\omega \theta} U( \theta(x), \omega(x)) & -\nabla_{\omega \omega} U( \theta(x), \omega(x)) 
		\end{bmatrix}  \right) dx \cdot \begin{bmatrix}
			 	\theta_t - \theta_* \\
				\omega_{t} - \omega_*
			 \end{bmatrix} \enspace.
	\end{align*}
	Assume that one can prove for some $r >0$, and $(\theta, \omega) \in B_2((\theta_*, \omega_*), r)$, with a proper choice of $\eta$, the largest singular value is bounded above by 1,  
	\begin{align*}
		\left\| I - \eta \begin{bmatrix}
					\nabla_{\theta\theta} U( \theta, \omega) & \nabla_{\theta \omega}U(\theta, \omega) \\
					-\nabla_{\omega \theta} U( \theta, \omega) & -\nabla_{\omega \omega} U( \theta, \omega) 
		\end{bmatrix} \right\|_{\rm op} < 1 \enspace.
	\end{align*}
	Then due to convexity of the operator norm, the dynamics of SGA is contracting locally because,
	\begin{align*}
		\left\| \begin{bmatrix}
			\theta_{t+1} - \theta_*\\
			\omega_{t+1} - \omega_*
		\end{bmatrix} \right\| 
		&\leq 
		\left\|
		\int_0^1 \left( I - \eta \begin{bmatrix}
					\nabla_{\theta\theta} U( \theta(x), \omega(x)) & \nabla_{\theta \omega} U( \theta(x), \omega(x)) \\
					-\nabla_{\omega \theta} U( \theta(x), \omega(x)) & -\nabla_{\omega \omega} U( \theta(x), \omega(x)) 
				\end{bmatrix}  \right) dx
		\right\|_{\rm op} \cdot
		\left\| \begin{bmatrix}
			\theta_{t} - \theta_*\\
			\omega_{t} - \omega_*
		\end{bmatrix} \right\| \enspace ,\\
		&\leq \int_0^1 \left\| \left( I - \eta \begin{bmatrix}
					\nabla_{\theta\theta} U( \theta(x), \omega(x)) & \nabla_{\theta \omega} U( \theta(x), \omega(x)) \\
					-\nabla_{\omega \theta} U( \theta(x), \omega(x)) & -\nabla_{\omega \omega} U( \theta(x), \omega(x)) 
				\end{bmatrix}  \right) \right\|_{\rm op} dx
		 \cdot
		\left\| \begin{bmatrix}
			\theta_{t} - \theta_*\\
			\omega_{t} - \omega_*
		\end{bmatrix} \right\|  \enspace,\\
		&<  \left\| \begin{bmatrix}
			\theta_{t} - \theta_*\\
			\omega_{t} - \omega_*
		\end{bmatrix} \right\| \enspace.
	\end{align*}
	
	Let's analyze the singular values of  
	$$
	I - \eta \begin{bmatrix}
						\nabla_{\theta\theta} U( \theta, \omega) & \nabla_{\theta \omega}U(\theta, \omega) \\
						-\nabla_{\omega \theta} U( \theta, \omega) & -\nabla_{\omega \omega} U( \theta, \omega) 
			\end{bmatrix} \enspace,
	$$
	assuming $\nabla_{\theta\theta} U( \theta, \omega) \succ 0$, $-\nabla_{\omega \omega} U( \theta, \omega) \succ 0$. Abbreviate
	\begin{align*}
		\begin{bmatrix}
				\nabla_{\theta\theta} U( \theta, \omega) & \nabla_{\theta \omega}U(\theta, \omega) \\
				-\nabla_{\omega \theta} U( \theta, \omega) & -\nabla_{\omega \omega} U( \theta, \omega) 
		\end{bmatrix} := 
		\begin{bmatrix}
			A & C\\
			-C^T & B
		\end{bmatrix} \enspace.
	\end{align*}
	The largest singular value of 
	$$
	I - \eta \begin{bmatrix}
						A & C\\
						-C^T & B
			\end{bmatrix} \enspace,
	$$
	is the square root of the largest eigenvalue of the following symmetric matrix
	\begin{align*}
		 \begin{bmatrix}
			I- \eta A & -\eta C\\
			\eta C^T & I - \eta B
		\end{bmatrix}
		 \begin{bmatrix}
			I- \eta A & \eta C\\
			-\eta C^T & I - \eta B
		\end{bmatrix} = 
		 \begin{bmatrix}
			(I- \eta A)^2 + \eta^2 CC^T &  -\eta^2 (AC-CB) \\
			-\eta^2 (C^TA - BC^T) & (I - \eta B)^2 + \eta^2 C^T C
		\end{bmatrix} \enspace.
	\end{align*}
	It is clear that when $\eta = 0$, the largest eigenvalue of the above matrix is $1$. Observe
	\begin{align*}
		 & \quad \begin{bmatrix}
			(I- \eta A)^2 + \eta^2 CC^T &  -\eta^2 (AC-CB) \\
			-\eta^2 (C^TA - BC^T) & (I - \eta B)^2 + \eta^2 C^T C
		\end{bmatrix} = I - 2\eta 
		\begin{bmatrix}
			A & 0 \\
			0 & B
		\end{bmatrix} + \eta^2 
		\begin{bmatrix}
			A^2 + CC^T & -AC+CB \\
			-C^TA+BC^T & B^2 + C^TC
		\end{bmatrix} \enspace, \\
		&  \prec \left[ 1 - 2\eta \lambda_{\min}\left( \begin{bmatrix}
			A & 0 \\
			0 & B
		\end{bmatrix} \right) + \eta^2 
		\lambda_{\max}\left( \begin{bmatrix}
			A^2 + CC^T & -AC+CB \\
			-C^TA+BC^T & B^2 + C^TC
		\end{bmatrix} \right) \right] I \enspace .
	\end{align*}
	If we choose $\eta$ to be
	\begin{align*}
		\eta = \frac{\min_{(\theta, \omega) \in B_2((\theta_*, \omega_*), r)}\lambda_{\min}\left( \begin{bmatrix}
			A_{\theta, \omega} & 0 \\
			0 & B_{\theta, \omega}
		\end{bmatrix} \right)}{\max_{(\theta, \omega) \in B_2((\theta_*, \omega_*), r)}\lambda_{\max}\left( \begin{bmatrix}
			A_{\theta, \omega}^2 + C_{\theta, \omega}C_{\theta, \omega}^T & -A_{\theta, \omega}C_{\theta, \omega}+C_{\theta, \omega}B_{\theta, \omega} \\
			-C_{\theta, \omega}^TA_{\theta, \omega}+B_{\theta, \omega}C^T_{\theta, \omega} & B_{\theta, \omega}^2 + C_{\theta, \omega}^TC_{\theta, \omega}
		\end{bmatrix} \right)} = \frac{\sqrt{\alpha}}{\beta} \enspace,
	\end{align*}
	then we find
	\begin{align*}
		\\\quad \begin{bmatrix}
			(I- \eta A)^2 + \eta^2 CC^T &  -\eta^2 (AC-CB) \\
			-\eta^2 (C^TA - BC^T) & (I - \eta B)^2 + \eta^2 C^T C
		\end{bmatrix}
				&\prec  \left( 1 - \frac{\alpha}{\beta} \right) I \enspace.
	\end{align*}
	In this case, 
	\begin{align*}
		\left\| \begin{bmatrix}
			\theta_{t+1} - \theta_*\\
			\omega_{t+1} - \omega_*
		\end{bmatrix} \right\| &\leq \sup_{(\theta, \omega) \in B_2((\theta_*, \omega_*), r)}\left\| I - \eta \begin{bmatrix}
					\nabla_{\theta\theta} U( \theta, \omega) & \nabla_{\theta \omega}U(\theta, \omega) \\
					-\nabla_{\omega \theta} U( \theta, \omega) & -\nabla_{\omega \omega} U( \theta, \omega) 
		\end{bmatrix} \right\|_{\rm op} \left\| \begin{bmatrix}
			\theta_{t} - \theta_*\\
			\omega_{t} - \omega_*
		\end{bmatrix} \right\| \enspace, \\
		& \leq \sqrt{1 - \frac{\alpha}{\beta}} \cdot \left\| \begin{bmatrix}
			\theta_{t} - \theta_*\\
			\omega_{t} - \omega_*
		\end{bmatrix} \right\| \enspace .
	\end{align*}
	Therefore, to obtain an $\epsilon$-minimizer one requires a number of steps equal to
	\begin{align*}
		2\frac{\beta}{\alpha} \log \frac{r}{\epsilon} \enspace.
	\end{align*}
\end{proof}

\begin{proof}[Proof of Remark~\ref{remark:GD}] Consider $U(\theta) = \frac{1}{2} \theta^T A \theta$ where $A \succ 0$ is strictly positive. Then gradient descent corresponds to $\theta_{t+1} = (I- \eta A)\theta_{t}$ and thus $\| \theta_{t+1} \| \leq \Vert I-\eta A \|_{\rm op} \| \theta_t \|$. Setting $\eta = 1/\lambda_{\max}(A)$ we have $I - \eta A \succeq 0$ so $\Vert I-\eta A \|_{\rm op} = \lambda_{\max}(I - \eta A) = 1-\lambda_{\min}(A) / \lambda_{\max}(A)$. Therefore $\Vert \theta_t \Vert \leq \Vert\theta_0\Vert\big[1-\lambda_{\min}(A)/\lambda_{\max}(A)\big]^t \leq \Vert \theta_0 \Vert e^{-t\lambda_{\min}(A)/\lambda_{\max}(A)}$. The number of iterations required to obtain an $\epsilon$-minimizer is thus bounded as $T \geq \frac{\lambda_{\max}(A)}{\lambda_{\min}(A)} \log \frac{r}{\epsilon}$.

\end{proof}

\begin{proof}[Proof of Corollary~\ref{cor:lower-bound}]
	We have 
	\begin{align*}
		\begin{bmatrix}
			\theta_{t+1} \\
			\omega_{t+1}
		\end{bmatrix} &=  \begin{bmatrix}
			\theta_t \\
			\omega_t
		\end{bmatrix} - \eta \begin{bmatrix}
			\nabla_{\theta} U(\theta_t, \omega_t) \\
			-\nabla_{\omega} U(\theta_t, \omega_t)
		\end{bmatrix} \enspace,
		 \\
		& =  \left( I - \eta \begin{bmatrix}
			I & C \\
			-C^T & I
		\end{bmatrix}  \right) \cdot \begin{bmatrix}
			 	\theta_t \\
				\omega_t
			 \end{bmatrix} \enspace.
	\end{align*}
	If we define $D =  \diag\big((1-\eta)^2 I + \eta^2 C C^T , (1-\eta)^2 I + \eta^2 C^T C \big)$ then using the Rayleigh quotient representation of $\lambda_{\min}(D)$ we obtain,
	\begin{align*}
		\left\| \begin{bmatrix}
			\theta_{t+1} \\
			\omega_{t+1}
		\end{bmatrix} \right\|^2 &=  
		\begin{bmatrix}
		 			 	\theta_t &
		 				\omega_t
		 \end{bmatrix}
		 D    
		\begin{bmatrix}
			 	\theta_t \\
				\omega_t
			 \end{bmatrix}
			 \geq \lambda_{\min}(D)
			 \left\| \begin{bmatrix}
			\theta_{t} \\
			\omega_{t}
		\end{bmatrix} \right\|^2
			  \enspace .
	\end{align*}
	On the other hand,
	\begin{align*}
		\lambda_{\min}(D) = \lambda_{\min}\left((1-\eta)^2 I + \eta^2 C C^T\right) = 1 - 2\eta + \big[1+ \lambda_{\min}(CC^T)\big] \eta^2 \geq \frac{\lambda_{\min}(CC^T)}{1+\lambda_{\min}(CC^T)}
	\end{align*}
	regardless of the choice of $\eta$, which proves the claim.
\end{proof}

\begin{proof}[Proof of Theorem~\ref{lem:unstable-omd}] 
		Recall that the OMD dynamics iteratively updates
		\begin{align*}
			\begin{bmatrix}
				\theta_{t+1} \\
				\omega_{t+1} 
			\end{bmatrix} = \left( I -  2\eta \begin{bmatrix}
				0& C \\
				-C^T & 0
			\end{bmatrix}  \right) \cdot \begin{bmatrix}
				 	\theta_t  \\
					\omega_t
					\end{bmatrix} +  \eta
			 \begin{bmatrix}
						0& C \\
						-C^T & 0
			\end{bmatrix}  \cdot \begin{bmatrix}
						 	\theta_{t-1}  \\
							\omega_{t-1}
			\end{bmatrix} \enspace.		
		\end{align*}
	Define the following matrices 
	\begin{align}
		R_1 = \frac{\left( I -  2\eta \begin{bmatrix}
				0& C \\
				-C^T & 0
			\end{bmatrix}  \right) + \left( I - 4 \eta^2 \begin{bmatrix}
				C C^T & 0 \\
				0 & C^T C \end{bmatrix} \right)^{1/2} }{2} \enspace, \\
		R_2 = \frac{\left( I -  2\eta \begin{bmatrix}
				0& C \\
				-C^T & 0
			\end{bmatrix}  \right) - \left( I - 4 \eta^2 \begin{bmatrix}
				C C^T & 0 \\
				0 & C^T C \end{bmatrix} \right)^{1/2} }{2} \enspace.
	\end{align}
	It is easy to verify that
	\begin{align*}
		R_1 + R_2 &= \left( I -  2\eta \begin{bmatrix}
				0& C \\
				-C^T & 0
			\end{bmatrix}  \right) \enspace, \\
		R_1 R_2 = R_2 R_1 &= \frac{\left( I -  2\eta \begin{bmatrix}
				0& C \\
				-C^T & 0
			\end{bmatrix}  \right)^2  - \left( I - 4 \eta^2 \begin{bmatrix}
				C C^T & 0 \\
				0 & C^T C \end{bmatrix} \right) }{4} = - \eta \begin{bmatrix}
						0& C \\
						-C^T & 0
			\end{bmatrix} \enspace.
	\end{align*}
	The commutative property $R_1 R_2 = R_2 R_1$ follows from a singular value decomposition argument: Letting
	$C = UDV^T$ be the SVD of $C$ ($D$ diagonal) one finds,
	\begin{align*}
		C \left( I - 4\eta^2 C^TC\right)^{1/2} = U D \left( I - 4\eta^2 D^2 \right)^{1/2} V^T = U \left( I - 4\eta^2 D^2 \right)^{1/2} D  V^T= \left( I - 4\eta^2 CC^T\right)^{1/2} C \enspace.
	\end{align*}
	Using the above equality, the commutative property follows
	\begin{align*}
		&\left( I -  2\eta \begin{bmatrix}
						0& C \\
						-C^T & 0
					\end{bmatrix}  \right)  \left( I - 4 \eta^2 \begin{bmatrix}
					C C^T & 0 \\
					0 & C^T C \end{bmatrix} \right)^{1/2} =  \left( I - 4 \eta^2 \begin{bmatrix}
					C C^T & 0 \\
					0 & C^T C \end{bmatrix} \right)^{1/2} \left( I -  2\eta \begin{bmatrix}
						0& C \\
						-C^T & 0
					\end{bmatrix}  \right) \enspace , \\
		&\Longrightarrow \quad R_1 R_2 = R_2 R_1 \enspace .
	\end{align*}

	Now we have the following relations for OMD,
		\begin{align*}
			\begin{bmatrix}
				\theta_{t+1} \\
				\omega_{t+1} 
			\end{bmatrix} - R_1  \begin{bmatrix}
				 	\theta_t  \\
					\omega_t
					\end{bmatrix} = R_2 \left( \begin{bmatrix}
				 	\theta_t  \\
					\omega_t
					\end{bmatrix} - R_1 \begin{bmatrix}
						 	\theta_{t-1}  \\
							\omega_{t-1}
			\end{bmatrix} \right) \enspace, \\
			\begin{bmatrix}
				\theta_{t+1} \\
				\omega_{t+1} 
			\end{bmatrix} - R_2  \begin{bmatrix}
				 	\theta_t  \\
					\omega_t
					\end{bmatrix} = R_1 \left( \begin{bmatrix}
				 	\theta_t  \\
					\omega_t
					\end{bmatrix} - R_2 \begin{bmatrix}
						 	\theta_{t-1}  \\
							\omega_{t-1}
			\end{bmatrix} \right) \enspace.
		\end{align*}
		Hence
		\begin{align}
			\label{eq:key}
			(R_1 - R_2) \begin{bmatrix}
				\theta_{t} \\
				\omega_{t} 
			\end{bmatrix} =  R_1^t \left( \begin{bmatrix}
				 	\theta_1  \\
					\omega_1
					\end{bmatrix} - R_2 \begin{bmatrix}
						 	\theta_{0}  \\
							\omega_{0}
			\end{bmatrix} \right) - 
			R_2^t \left( \begin{bmatrix}
							 	\theta_1  \\
								\omega_1
								\end{bmatrix} - R_1 \begin{bmatrix}
									 	\theta_{0}  \\
										\omega_{0}
						\end{bmatrix} \right) \enspace .
		\end{align}
	
		Let's analyze the singular values of $R_1$ and $R_2$. We have,
		\begin{align*}
			R_1 &= \frac{\left( I -  2\eta \begin{bmatrix}
				0& C \\
				-C^T & 0
			\end{bmatrix}  \right) + \left( I - 4 \eta^2 \begin{bmatrix}
				C C^T & 0 \\
				0 & C^T C \end{bmatrix} \right)^{1/2} }{2} \enspace, \\
				&= \begin{bmatrix}
				\frac{I + (I-4\eta^2 CC^T)^{1/2}}{2} & -\eta C \\
				\eta C^T  & \frac{I + (I-4\eta^2 C^T C)^{1/2}}{2} \end{bmatrix} \enspace, \\
			R_1 R_1^T &=  \begin{bmatrix}
				\frac{I + (I-4\eta^2 CC^T)^{1/2}}{2} & -\eta C \\
				\eta C^T  & \frac{I + (I-4\eta^2 C^T C)^{1/2}}{2} \end{bmatrix} \begin{bmatrix}
				\frac{I + (I-4\eta^2 CC^T)^{1/2}}{2} & \eta C \\
				-\eta C^T  & \frac{I + (I-4\eta^2 C^T C)^{1/2}}{2} \end{bmatrix} \enspace , \\
				& = \begin{bmatrix}
				\left( \frac{I + (I-4\eta^2 CC^T)^{1/2}}{2} \right)^2 + \eta^2 CC^T & 0 \\
				0  & \left( \frac{I + (I-4\eta^2 C^T C)^{1/2}}{2} \right)^2 + \eta^2 C^TC \end{bmatrix} \enspace , \\
				& = \begin{bmatrix}
				\frac{I + (I-4\eta^2 CC^T)^{1/2}}{2} & 0 \\
				0  & \frac{I + (I-4\eta^2 C^T C)^{1/2}}{2} \end{bmatrix} \enspace.
		\end{align*}
		Similarly
		\begin{align*}
				R_2 R_2^T =  \begin{bmatrix}
					\frac{I - (I-4\eta^2 CC^T)^{1/2}}{2} & 0 \\
					0  & \frac{I - (I-4\eta^2 C^T C)^{1/2}}{2} \end{bmatrix} \enspace.
		\end{align*}
		
		For $\eta$ small enough, the spectral radius satisfies the strict inequality $\| R_1 \|_{\rm op} < 1$. Concretely, for example,
		\begin{align*}
			&\eta = \frac{1}{2\sqrt{2\lambda_{\max}(CC^T)}} \implies \\
			&\frac{I + (I-4\eta^2 CC^T)^{1/2}}{2} \preceq \frac{I + (I-2\eta^2 CC^T)}{2} = I - \eta^2 CC^T \preceq \left[ 1 - \frac{1}{8} \frac{\lambda_{\min}(CC^T)}{\lambda_{\max}(CC^T)} \right] I \enspace, \\
			& \frac{I - (I-4\eta^2 CC^T)^{1/2}}{2} \preceq \frac{1}{2}I \enspace.
		\end{align*}
		Therefore $\| R_1 \|_{\rm op} \leq \sqrt{1 - \frac{1}{8} \frac{\lambda_{\min}(CC^T)}{\lambda_{\max}(CC^T)}}, ~~\| R_2\|_{\rm op} \leq \sqrt{1 - \frac{1}{2}}$. 
		
		The RHS of Eqn.~\eqref{eq:key} is upper bounded because
		\begin{align*}
		 \left\| R_2^t \left( \begin{bmatrix}
							 	\theta_1  \\
								\omega_1
								\end{bmatrix} - R_1 \begin{bmatrix}
									 	\theta_{0}  \\
										\omega_{0}
						\end{bmatrix} \right) \right\| &\leq \left(1-\frac{1}{2}\right)^{t/2} (\| (\theta_1, \omega_1)\| + \| (\theta_0, \omega_0)\|) \enspace ,\\
		 \left\| R_1^t \left( \begin{bmatrix}
							 	\theta_1  \\
								\omega_1
								\end{bmatrix} - R_2 \begin{bmatrix}
									 	\theta_{0}  \\
										\omega_{0}
						\end{bmatrix} \right) \right\| &\leq \left( 1 - \frac{1}{8} \frac{\lambda_{\min}(CC^T)}{\lambda_{\max}(CC^T)} \right)^{t/2} (\| (\theta_1, \omega_1)\| + \| (\theta_0, \omega_0)\|) \enspace.
		\end{align*}
		Moreover the LHS satisfies
		\begin{align*}
			\left\| (R_1 - R_2) \begin{bmatrix}
				\theta_{t} \\
				\omega_{t} 
			\end{bmatrix} \right\| \geq \left\| \begin{bmatrix}
				\theta_{t} \\
				\omega_{t} 
			\end{bmatrix} \right\| \sqrt{\frac{1}{2}} \enspace .
		\end{align*}
		Combining these inequalities we obtain
\begin{equation}
 \Vert (\theta_t, \omega_t) \Vert \sqrt{\frac{1}{2}} \leq \left[\left(1-\frac{1}{2}\right)^{t/2} + \left(1-\frac{1}{8}\frac{\lambda_{\rm min}(CC^T)}{\lambda_{\rm max}(CC^T)}\right)^{t/2}\right](\Vert (\theta_0, \omega_0) \Vert + \Vert (\theta_1, \omega_1) \Vert) \enspace .
\end{equation}
By our assumption $\Vert (\theta_0, \omega_0) \Vert , \Vert (\theta_1, \omega_1) \Vert \leq r$ so
\begin{align}
 \Vert (\theta_t, \omega_t) \Vert
 	& \leq 2\sqrt{2} r \left[\left(1-\frac{1}{2}\right)^{t/2} + \left(1-\frac{1}{8}\frac{\lambda_{\rm min}(CC^T)}{\lambda_{\rm max}(CC^T)}\right)^{t/2}\right] \enspace , \\
	& \leq  2\sqrt{2} r \left[\exp\left(-\frac{t}{4}\right) + \exp\left(-\frac{t}{16}\frac{\lambda_{\rm min}(CC^T)}{\lambda_{\rm max}(CC^T)}\right)\right] \enspace ,
\end{align}
because $\forall (x,\alpha) \in \mathbb{R} \times \mathbb{R}^+ : (1-x)^\alpha \leq e^{-\alpha x} $. Thus
\begin{align}
\Vert (\theta_t, \omega_t) \Vert
	& \leq  2\sqrt{2} r \exp\left(-\frac{t}{16}\frac{\lambda_{\rm min}(CC^T)}{\lambda_{\rm max}(CC^T)}\right) \left[1 + \exp\left(-\frac{t}{4}+\frac{t}{16}\frac{\lambda_{\rm min}(CC^T)}{\lambda_{\rm max}(CC^T)}\right)\right] \enspace .
\end{align}
But $16 \geq 4\frac{\lambda_{\rm min}(CC^T)}{\lambda_{\rm max}(CC^T)} \implies -\frac{1}{4}+\frac{1}{16}\frac{\lambda_{\rm min}(CC^T)}{\lambda_{\rm max}(CC^T)} \leq 0$ so
\begin{align}
\Vert (\theta_t, \omega_t) \Vert
	& \leq 2\sqrt{2} r \exp\left(-\frac{t}{16}\frac{\lambda_{\rm min}(CC^T)}{\lambda_{\rm max}(CC^T)}\right) \left[1 + 1\right] \enspace , \\
	& =  4\sqrt{2} r \exp\left(-\frac{t}{16}\frac{\lambda_{\rm min}(CC^T)}{\lambda_{\rm max}(CC^T)}\right) \enspace .
\end{align}
		To sum up, when
		$$
		T \geq \left\lceil 16\frac{\lambda_{\max}(CC^T)}{\lambda_{\min}(CC^T)} \log \frac{4\sqrt{2} r}{\epsilon} \right\rceil \enspace ,
		$$
		one can ensure $\| (\theta_T, \omega_T) \|\leq \epsilon$. 
\end{proof}

	\begin{proof}[Proof of Theorem~\ref{lem:unstable-co}]
		In this case, the consensus optimization satisfies the following update
		\begin{align}
			\begin{bmatrix}
				\theta_{t+1} \\
				\omega_{t+1}
			\end{bmatrix} = \left( I -  \eta \begin{bmatrix}
				\gamma CC^T & C \\
				-C^T & \gamma C^T C
			\end{bmatrix}  \right) \cdot \begin{bmatrix}
				 	\theta_t  \\
					\omega_t
					\end{bmatrix} \enspace .
		\end{align}
		Again let's analyze the singular values of the operator $K := \left( I -  \eta \begin{bmatrix}
				\gamma CC^T & C \\
				-C^T & \gamma C^T C
			\end{bmatrix}  \right)$, or equivalently, the eigenvalues of $KK^T$,
		\begin{align*}
			KK^T &= \begin{bmatrix}
				I - \eta\gamma CC^T & -\eta C \\
				\eta C^T & I - \eta\gamma C^T C
			\end{bmatrix} \begin{bmatrix}
				I - \eta\gamma CC^T & \eta C \\
				-\eta C^T & I - \eta\gamma C^T C
			\end{bmatrix} \enspace,	\\
			&= 	\begin{bmatrix}
				(I - \eta\gamma CC^T)^2 + \eta^2 CC^T & (I - \eta\gamma CC^T)\eta C - \eta C (I - \eta\gamma C^TC)\\
				\eta C^T (I - \eta\gamma CC^T) - (I - \eta\gamma C^TC) \eta C^T & (I - \eta\gamma C^T C)^2 + \eta^2 C^T C
			\end{bmatrix} \enspace, \\
			& = \begin{bmatrix}
				(I - \eta\gamma CC^T)^2 + \eta^2 CC^T & 0\\
				0 & (I - \eta\gamma C^T C)^2 + \eta^2 C^T C
			\end{bmatrix} \enspace.
		\end{align*}
		Now consider the largest eigenvalue of $(I - \eta\gamma CC^T)^2 + \eta^2 CC^T$, for a fixed $\gamma$, with a properly chosen $\eta$. Using the SVD $C = UDV^T$, we obtain
		\begin{align*}
			&(I - \eta\gamma CC^T)^2 + \eta^2 CC^T = U\left[ (I - \eta\gamma D^2)^2 + \eta^2 D^2 \right] U^T\\
			& \preceq \left[ 1 - 2\gamma  \lambda_{\min}(CC^T) \eta + (\gamma^2 \lambda^2_{\max}(CC^T) + \lambda_{\max}(CC^T)  \eta^2 \right] I \enspace, \\
			& = \left[  1 -  \frac{\gamma^2 \lambda^2_{\min}(CC^T)}{\gamma^2 \lambda^2_{\max}(CC^T) + \lambda_{\max}(CC^T)}  \right] I \enspace,
		\end{align*}
		with 
		$$
		\eta = \frac{\gamma \lambda_{\min}(CC^T)}{\lambda_{\max}(CC^T) + \gamma^2 \lambda^2_{\max}(CC^T)} \enspace.
		$$
	\end{proof}

\begin{proof}[Proof of Theorem~\ref{lem:unstable-implicit}]
	In this case, the implicit update satisfies the update rule
	\begin{align}
		\begin{bmatrix}
			\theta_{t+1} \\
			\omega_{t+1}
		\end{bmatrix} &= \begin{bmatrix}
			 	\theta_t  \\
				\omega_t
				\end{bmatrix}  -  \eta \begin{bmatrix}
			0 & C \\
			-C^T & 0
		\end{bmatrix} \cdot \begin{bmatrix}
			\theta_{t+1} \\
			\omega_{t+1}
		\end{bmatrix} \\
		\left( I + \eta \begin{bmatrix}
			0 & C \\
			-C^T & 0
		\end{bmatrix} \right) \cdot \begin{bmatrix}
			\theta_{t+1} \\
			\omega_{t+1}
		\end{bmatrix}  &=  \begin{bmatrix}
			 	\theta_t  \\
				\omega_t
				\end{bmatrix}\enspace . \label{eq:IU}
	\end{align}
	Let's analyze the singular values of the matrix $K:= \left( I + \eta \begin{bmatrix}
			0 & C \\
			-C^T & 0
		\end{bmatrix} \right)$, or equivalently the root of eigenvalues of $KK^T$
		\begin{align*}
			KK^T &=  \begin{bmatrix}
			I & \eta C \\
			-\eta C^T & I
		\end{bmatrix} \begin{bmatrix}
			I & -\eta C \\
			\eta C^T & I
		\end{bmatrix} \\
		&= \begin{bmatrix}
			I + \eta^2 CC^T & 0 \\
			0 & I + \eta^2 C^T C
		\end{bmatrix} \enspace .
		\end{align*}
		It is clear that the singular values of $K$, denoted by $\sigma_i(K)$ is sandwiched between 
		\begin{align*}
			\sqrt{1 + \eta^2 \lambda_{\min}(CC^T)} \leq \sigma_i(K) \leq \sqrt{1 + \eta^2 \lambda_{\max}(CC^T)} \enspace.
		\end{align*}
		If we choose $\eta = \frac{1}{\sqrt{\lambda_{\max}(CC^T)}}$, then
		\begin{align*}
			0 \leq \eta^2 \lambda_{\min}(CC^T) \leq \frac{\lambda_{\min}(CC^T)}{\lambda_{\max}(CC^T)} \leq 1 \enspace.
		\end{align*}
		Using the fact that for all $0\leq t\leq 1$, $1 + (\sqrt{2}-1) t \leq \sqrt{1+t}$, then
		\begin{align*}
			 \sigma_{\min}(K) &\geq \sqrt{1 + \eta^2 \lambda_{\min}(CC^T)} \geq 1 + (\sqrt{2}-1) \frac{\lambda_{\min}(CC^T)}{\lambda_{\max}(CC^T)} \enspace.
		\end{align*}
		Recall Eqn.~\eqref{eq:IU}, using the fact for all $0\leq t\leq 1$, $1/(1+(\sqrt{2}-1)t) \leq 1 - (1-1/\sqrt{2})t$,  we know
		\begin{align*}
			\sigma_{\min}(K) \cdot \left\| \begin{bmatrix}
			\theta_{t+1} \\
			\omega_{t+1}
		\end{bmatrix} \right\| & \leq \left\| \begin{bmatrix}
			\theta_{t} \\
			\omega_{t}
		\end{bmatrix} \right\| \\
		\left\| \begin{bmatrix}
					\theta_{t+1} \\
					\omega_{t+1}
				\end{bmatrix} \right\| & \leq \frac{1}{1 + (\sqrt{2}-1) \frac{\lambda_{\min}(CC^T)}{\lambda_{\max}(CC^T)}} \left\| \begin{bmatrix}
			\theta_{t} \\
			\omega_{t}
		\end{bmatrix} \right\| \\
		& \leq \left( 1 - (1-\frac{1}{\sqrt{2}})  \frac{\lambda_{\min}(CC^T)}{\lambda_{\max}(CC^T)} \right) \left\| \begin{bmatrix}
			\theta_{t} \\
			\omega_{t}
		\end{bmatrix} \right\| \enspace .
		\end{align*}
		To sum up, when
		$$
		T \geq \left\lceil (2+\sqrt{2})\frac{\lambda_{\max}(CC^T)}{\lambda_{\min}(CC^T)}  \log \frac{r}{\epsilon} \right\rceil \enspace ,
		$$
		one can ensure $\| (\theta_T, \omega_T) \|\leq \epsilon$.
\end{proof}

\begin{proof}[Proof of Theorem~\ref{lem:unstable-pm}]
	In the simple bi-linear game case,
		\begin{align*}
			\begin{bmatrix}
				\theta_{t+1} \\
				\omega_{t+1} 
			\end{bmatrix} &= \begin{bmatrix}
				\theta_{t} \\
				\omega_{t} 
			\end{bmatrix} - \begin{bmatrix}
				0 & \eta C \\
				-\eta C^T & 0
			\end{bmatrix} \begin{bmatrix}
				\theta_{t+1/2} \\
				\omega_{t+1/2} 
			\end{bmatrix} \enspace, \\
			& = \begin{bmatrix}
				\theta_{t} \\
				\omega_{t} 
			\end{bmatrix} - \begin{bmatrix}
				0 & \eta C \\
				-\eta C^T & 0
			\end{bmatrix} \begin{bmatrix}
				I & -\gamma C \\
				\gamma C^T & I
			\end{bmatrix} \begin{bmatrix}
				\theta_{t} \\
				\omega_{t} 
			\end{bmatrix} \enspace , \\
			&=  \begin{bmatrix}
				I-\eta \gamma CC^T & -\eta C \\
				\eta C^T & I-\eta \gamma C^T C
			\end{bmatrix}
			\begin{bmatrix}
				\theta_{t} \\
				\omega_{t} 
			\end{bmatrix} \enspace.
		\end{align*}
		Note this linear system is the same as that in Thm.~\ref{lem:unstable-co}. Therefore the convergence analysis follows in the same way as Thm.~\ref{lem:unstable-co}. 
\end{proof}

\end{document}